\def\year{2021}\relax
\documentclass[letterpaper]{article} 
\usepackage{aaai21}  
\usepackage{times}  
\usepackage{helvet} 
\usepackage{courier}  
\usepackage[hyphens]{url}  
\usepackage{graphicx} 
\usepackage{natbib}  
\usepackage{caption} 
\usepackage{amsfonts}
\usepackage{amsmath}
\usepackage{amsthm}
\usepackage{algorithm}
\usepackage{algorithmic}
\usepackage{xspace}
\usepackage{graphicx}
\usepackage{subfigure}
\usepackage{xcolor}
\usepackage{cite}

\newcommand{\ouralg}{$\mathsf{SimpleUCB}$\xspace}

\newcommand{\ouralgtwo}{$\mathsf{MaxMinUCB}$\xspace}

\newcommand{\ouralgthree}{$\mathsf{MinWD}$\xspace}

\theoremstyle{plain}
\newtheorem{theorem}{Theorem}
\newtheorem{corollary}{Corollary}[theorem]
\newtheorem{lemma}[theorem]{Lemma}
\newtheorem{proposition}[theorem]{Proposition}
\theoremstyle{definition}
\newtheorem{definition}{Definition}
\newtheorem{remark}{Remark}

\title{Robust Bandit Learning with Imperfect Context}

\author{
	Jianyi Yang,\; Shaolei Ren\\
}
\affiliations{
	University of California, Riverside\\
	
	{\{jyang239, shaolei\}@ucr.edu}
	
}
\begin{document}

\maketitle

\begin{abstract}
	A standard assumption in contextual multi-arm bandit is that the true context
	is perfectly known 
	before arm selection.
	Nonetheless, in many practical applications
	(e.g., cloud resource management),
	prior to arm selection, the context information can only
	be acquired by prediction 
	subject to
	errors or adversarial modification.
	In this paper, we study a contextual bandit setting
	in which only imperfect context is available for arm selection
	while the true context is revealed at the end of each round.
	We propose two robust
	arm selection
	algorithms: \ouralgtwo (Maximize Minimum UCB) which  maximizes the worst-case reward,
	and \ouralgthree (Minimize Worst-case Degradation) which  minimizes the worst-case regret.
	Importantly,
	we analyze the robustness of \ouralgtwo and \ouralgthree
	by deriving  both regret and reward bounds
	compared to  an
	oracle that knows the true context.
	Our results show that as time goes on, \ouralgtwo
	and \ouralgthree both perform as  asymptotically well
	as their optimal counterparts that know the reward function.
	Finally, we apply \ouralgtwo and \ouralgthree to
	online edge datacenter selection, and run synthetic simulations to validate
	our theoretical analysis.
\end{abstract}

\section{Introduction}\label{sec:Introduction}

Contextual bandits \cite{Lu2010,chu2011LinUCB} concern online
learning scenarios such as recommendation systems \cite{Bandit_LinUCB_WWW_2010_Li:2010:CAP:1772690.1772758},  mobile health \cite{Lei2014}, cloud resource provisioning \cite{Xu2019}, wireless
communications \cite{Bandit_Contextual_LinkAdaptation_Wireless_Berkeley_2019_10.1145/3341216.3342212}, economics\cite{pourbabaee2020robust},  in which
arms (a.k.a., actions) are selected based on the underlying context to
balance the tradeoff between exploitation
of the already learnt knowledge and exploration
of uncertain arms \cite{AuerEXP3,AuerUCB,Bubeck2012,Dani2008}.

The majority of the existing studies on contextual bandits \cite{chu2011LinUCB,Valko13,Bandit_Contextual_LinkAdaptation_Wireless_Berkeley_2019_10.1145/3341216.3342212} assume that a perfectly accurate context is known before each arm selection. Consequently, as long as the agent learns increasingly more knowledge about reward, it can select arms with lower and lower average regrets.
In many cases, however, the perfect (or true) context is not available to the
agent prior to arm selection. Instead, the true context is  revealed  after taking an action at the end of each round \cite{Bandit_ContextualDistribution_ETHZ_NIPS_2019_Kirschner2019StochasticBW},
but can be predicted using predictors,
such as time series prediction\cite{time_series_prediction_Brockwell2016,LSTM_gers1999},  to facilitate the agent's arm selection.
For example,
in wireless communications,
the channel condition is subject
to various attenuation effects (e.g., path loss and small-scale
multi-path fading), and is critical context information for
the transmitter configuration such as modulation and rate adaption (i.e., arm selection)  \cite{Goldsmith_WirelessCommunications_2005,Bandit_Contextual_LinkAdaptation_Wireless_Berkeley_2019_10.1145/3341216.3342212}.
But, the channel
condition context is
predicted and hence can only be coarsely
known until the completion of transmission.
For another example,
the exact workload
arrival rate
is crucial context information for  cloud resource management, but
cannot
be known until the workload actually arrives.
Naturally,  context prediction
is subject to prediction errors.
Moreover,
it can also open a new attack surface ---
an outside attacker may adversarially modify the predicted context. For example, a recent study \cite{Baosen_Adversarial_LoadForecasting_SmartGrid_eEnergy_2019_10.1145/3307772.3328314}
shows that the energy load predictor in smart grid can be adversarially attacked to produce load estimates with higher-than-usual errors. More motivating examples are provided in \cite{yang2021robust}.
In general, imperfectly predicted and even adversarially presented
context is very common in practice.

As motivated by practical problems,
we consider a bandit setting
where the agent receives imperfectly predicted context
and selects an arm at the beginning of each round and the context is revealed after arm selection. We focus on robust arm optimization given imperfect context, which is as crucial as robust reward function estimation or exploration 
in contextual bandits \cite{doubly_robust_bandit_dudik2011,Adversarial_Linear_Bandit_Neu2020,robust_contextual_bandit_mhealth_zhu2018}.
Concretely, with imperfect context, our goal is to select arms online in a robust manner to optimize
the worst-case performance in a neighborhood domain with the received imperfect context as center and a defense budget as radius. In this way, the robust arm selection can defend against the imperfect context error ( from either context prediction error or adversarial modification) constrained by the budget.

Importantly and interestingly, given imperfect context,
maximizing the worst-case reward (referred to as type-I robustness objective) and minimizing
the worst-case regret (referred to as type-II robustness objective) can lead to different arms,
while they are the same under the setting of perfect context \cite{Bandit_Contextual_LinkAdaptation_Wireless_Berkeley_2019_10.1145/3341216.3342212,Bandit_LinUCB_WWW_2010_Li:2010:CAP:1772690.1772758,Bandit_Introduction_Booklet_2019_MAL-068}. Given imperfect context, the strategy for type-I robustness is more conservative than that for type-II robustness in terms of reward. The choice of the robustness objective depends on applications. For example, some safety-aware applications \cite{Safety_aware_bandits_sun2017, safe_RL_survey_garcia2015} intend to avoid extremely low reward, and thus type-I objective is suitable for them. Other applications \cite{Bandit_LinUCB_WWW_2010_Li:2010:CAP:1772690.1772758,edge_service_placement_bandit_chen2018,robust_bandit_attack_guan2020} focus on preventing large sub-optimality of selected arms, and type-II objective is more appropriate. As a distinction from other works on robust optimization of bandits \cite{Bandit_RobustOptimization_GaussianProcess_NIPS_2018_10.5555/3327345.3327478,Bandit_RobustDistributionalContextual_AISTATS_2020_kirschner2020distributionally,DRBQO_nguyen2020distributionally}, we highlight the difference of the two types of robustness objectives.

We derive
two algorithms --- \ouralgtwo (Maximize Minimum UCB), which maximizes the worst-case reward for type-I objective,
and \ouralgthree (Minimize Worst-case Degradation), which minimizes the worst-case regret for type-II objective.
The challenge of algorithm designs is that the agent has no access to exact knowledge of reward function but the estimated counterpart based on history collected data. Thus, in our design, \ouralgtwo maximizes the lower bound of reward,  while \ouralgthree minimizes the upper bound of regret.

We analyze the robustness of \ouralgtwo and \ouralgthree
by deriving  both regret and reward bounds,
compared to a strong
oracle that knows
the true context for arm selection as well as the
exact reward function.
Importantly,
our results show that,
while a linear regret term exists
for both \ouralgtwo and \ouralgthree
due to imperfect context,
the added linear regret term is
actually the same as the amount of regret
incurred by respectively optimizing type-I and type-II objectives
with perfect knowledge of the reward function. This
implies that as time goes on, \ouralgtwo and \ouralgthree will asymptotically approach the corresponding optimized objectives from the reward and regret views, respectively.

Finally, we apply \ouralgtwo and \ouralgthree to the
problem of online edge datacenter selection
and run synthetic simulations to validate
our theoretical analysis.

\section{Related Work}\label{sec:RelatedWork}

\textbf{Contextual bandits.}
Linear contextual bandit learning is considered in LinUCB by \cite{Bandit_LinUCB_WWW_2010_Li:2010:CAP:1772690.1772758}. .
The study \cite{yadkori2011} improves the regret analysis of linear contextual bandit learning,
while the studies \cite{agrawal12,agrawal13} solve this problem by Thompson sampling and give a regret bound. 
There are also studies to extend the algorithms to general reward
functions like non-linear functions, for which kernel method is exploited in GP-UCB \cite{Srinivas10}, Kernel-UCB \cite{Valko13}, IGP-UCB and GP-TS \cite{Chowdhury17,Deshmukh17}. 
Nonetheless, a standard assumption in these studies is
that perfect context is available for arm selection,
whereas imperfect context is common in many practical applications \cite{Bandit_RobustDistributionalContextual_AISTATS_2020_kirschner2020distributionally}.

\textbf{Adversarial bandits and Robustness.}
The prior studies on adversarial bandits \cite{Bandit_Adversarial_Auer_COLT_2016_pmlr-v49-auer16,Bandit_AdversarialAttacks_LihongLi_Google_NIPS_2018_10.5555/3327144.3327281,Bandit_ContaminatedBandit_JMLR_2019_JMLR:v20:18-395,Bandit_DataPoison_NessShroff_ICML_2019_10.5555/3327144.3327281}
have primarily focused on that the adversary maliciously presents
rewards to the agent or directly injects errors in
rewards.
Moreover, many studies \cite{Minimax_adversarial_bandit_Audibert2009_COLT,Adversarial_bandits_lowerbounds_gerchinovitz2016_NIPS}  consider the best constant policy throughout the entire learning process as
the oracle, while in our setting the best arm 
depends on the true context at each round.  The adversarial setting has also been extended to contextual bandits \cite{Adversarial_Linear_Bandit_Neu2020,Adversarial_contextual_bandit_syrgkanis2016_ICML,Sequential_Batch_bandit_han2020sequential}.

Recently, robust bandit algorithms have been proposed for various adversarial settings. Some focus on robust reward estimation and exploration \cite{Bandit_ContaminatedBandit_JMLR_2019_JMLR:v20:18-395, robust_bandit_attack_guan2020,doubly_robust_bandit_dudik2011}, and others train a robust or distributionally robust policy \cite{conservative_bandits_wu2016, Adversarial_contextual_bandit_syrgkanis2016_ICML,Distributionally_robust_policy_si2020,Distributionally_robust_batch_bandit_si2020}.
Our study
differs from the existing adversarial bandits by seeking
two different robust algorithms given imperfect (and possibly adversarial) context. 

\textbf{Optimization and bandits with imperfect context.} 
\cite{online_learning_rakhlin2013} considers online optimization with predictable sequences and \cite{online_optimization_dynamix_jadbabaie2015} focuses on adaptive online optimization competing with dynamic benchmarks. Besides, \cite{Robust_optimization_minimax_regret_chen2014robust,Two-stage_minimax_regret_jiang2013two} study the robust optimization of mini-max regret. These studies assume perfectly known cost functions
without learning. 
A recent study \cite{Bandit_RobustOptimization_GaussianProcess_NIPS_2018_10.5555/3327345.3327478}
considers Bayesian optimization and aims at identifying a worst-case good input region
with input perturbation (which can
also model a perturbed but fixed environment/context parameter).
The study \cite{Wang16} considers the linear bandit where certain context features are hidden, and uses iterative methods to estimate
hidden contexts and model parameters.
Another recent study \cite{Bandit_ContextualDistribution_ETHZ_NIPS_2019_Kirschner2019StochasticBW}
assumes the knowledge
of context distribution for arm selection,
and considers a weak oracle that also only knows context distribution.
The relevant papers \cite{Bandit_RobustDistributionalContextual_AISTATS_2020_kirschner2020distributionally} and \cite{DRBQO_nguyen2020distributionally}  consider robust Bayesian optimizations where context distribution information
is imperfectly provided, and propose to maximize the worst-case expected reward for distributional robustness.
Although the objective of \ouralgtwo in our paper is similar to the robust optimization objectives in the two papers, we additionally derive a lower bound for the true reward in our analysis,
which provides another perspective on the robustness of arm selection. More importantly, considering that the objectives in the two relevant papers are equivalent to minimizing a pseudo robust regret, we propose \ouralgthree  and derive an upper bound for the incurred true regret.

\section{Problem Formulation}

Assume that at the beginning of round $t$, the agent receives
imperfect context $\hat{x}_t\in\mathcal{X}$ which is exogenously provided and not necessarily the true context $x_t$.
Given the imperfect context $\hat{\mathbf{x}}_t\in\mathcal{X}$
and an arm set $\mathcal{A}$,
the agent selects an arm $a_t\in\mathcal{A}$ for round $t$.
Then, the reward $y_t$ along with the true context $x_t$ is revealed to the agent at
the end of round $t$. Assume that $\mathcal{X}\times \mathcal{A}\subseteq \mathbb{R}^d$, and we use $x_{a_t,t}$ to denote the $d$-dimensional concatenated vector $[x_t, a_t]$.

The reward $y_t$ received by the agent in round $t$ is jointly decided by
the true context $x_t$ and selected arm $a_t$, and
can be expressed as follows
\begin{equation}\label{observereward_batch}
y_{t}=f(x_{t},a_t)+n_{t},
\end{equation}
where $f:\mathcal{X}\times \mathcal{A}\rightarrow \mathbb{R}$ is the reward function, $\mathcal{X}$ is the context domain,
and $n_t$ is the noise term. We assume
that the reward function $f$ belongs to a reproducing kernel Hilbert space (RKHS) $\mathcal{H}$ generated by a kernel function $k: \left( \mathcal{X}\times \mathcal{A}\right)\times\left( \mathcal{X}\times \mathcal{A}\right)\ \rightarrow \mathbb{R}$. In this RKHS,
there exists a mapping function $\phi: \left( \mathcal{X}\times \mathcal{A}\right)\rightarrow \mathcal{H}$ which maps context and arm to their corresponding feature in $\mathcal{H}$. By reproducing property, we have $k\left([x,a],[x',a'] \right) =\left\langle \phi\left( x,a\right),\phi\left( x',a'\right) \right\rangle $ and $f\left(x,a \right) =\left\langle \phi\left( x,a\right),\theta\right\rangle $ where
$\theta$ is the representation of function $f(\cdot,\cdot)$ in $\mathcal{H}$. Further, as commonly considered in the bandit literature \cite{Bandit_Introduction_Booklet_2019_MAL-068,Bandit_LinUCB_WWW_2010_Li:2010:CAP:1772690.1772758},
the noise
$n_t$ follows $b$-sub-Gaussian distribution for a constant $b\geq 0$, i.e. conditioned on the filtration $\mathcal{F}_{t-1}=\left\{ x_{\tau}, y_{a,\tau},a_{\tau}, \tau=1,\cdots,t-1\right\}$, $\forall \sigma\in \mathbb{R}$,
$
\mathbb{E}\left[e^{\sigma n_t}|\mathcal{F}_{t-1} \right] \leq \exp\left(\frac{\sigma^2b^2}{2} \right).
$

Without knowledge of reward function $f$, bandit algorithms are designed to decide an arm sequence $\left\lbrace a_t, t=1,\cdots,T\right\rbrace $ to minimize the cumulative regret
\begin{equation}\label{eqn:trueregret}
R_T=\sum_{t=1}^Tf(x_{t},A^*(x_{t}))-f(x_{t},a_t),
\end{equation}
where $A^{*}\left(x_t \right) =\arg\max_{a\in\mathcal{A}}f(x_t,a)$ is the oracle-optimal arm  at round $t$
given the true context $x_t$. When the received contexts are perfect, i.e. $\hat{x}_t=x_t$, minimizing the cumulative regret is equivalent to maximizing the cumulative reward
$
F_T=\sum_{t=1}^Tf(x_{t},a_t).
$

\subsection{Context Imperfectness}
The context
error can come from a variety
of sources, including imperfect context prediction algorithms
and adversarial
corruption \cite{Bandit_RobustDistributionalContextual_AISTATS_2020_kirschner2020distributionally,Baosen_Adversarial_LoadForecasting_SmartGrid_eEnergy_2019_10.1145/3307772.3328314} on context.
We simply use context error to
encapsulate all the error sources without further differentiation.
We assume that context error $\left\|x_t-\hat{x}_t \right\|$, where $\|\cdot\|$ is a certain norm  \cite{Bandit_RobustOptimization_GaussianProcess_NIPS_2018_10.5555/3327345.3327478}, is less than $\Delta$.  Also, $\Delta$ is referred to as the defense \emph{budget} and can be considered as the level
of robustness/safeguard that the agent intends to provide
against context errors: with a larger $\Delta$,
the agent wants to make its arm selection robust against
larger context errors (at the possible expense of its reward).
A time-varying error budget can be captured by using $\Delta_t$.  Denote the neighborhood domain of context $x$ as
$
\mathcal{B}_{\Delta}\left(x \right)  =\left\lbrace y\in\mathcal{X} \mid \left\|y-x \right\|\leq \Delta\right\rbrace
$.
Then, we have the true context $x_t\in \mathcal{B}_{\Delta}\left(\hat{x}_t \right)$,
where $\hat{x}_t$ is available to the agent.

\subsection{Reward Estimation}
Reward estimation is critical for arm selection. Kernel ridge regression, which is widely used in contextual bandits \cite{Bandit_Introduction_Booklet_2019_MAL-068} serves as the reward estimation method in our algorithm designs.
By kernel ridge regression, the estimated reward given arm $a$ and context $x$ is expressed as
\begin{equation}\label{equ:estimean}
\hat{f}_{t}(x,a)=\mathbf{k}_{t}^T(x,a)(\mathbf{K}_{t}+\lambda \mathbf{I})^{-1}\mathbf{y}_{t}
\end{equation}
where $\mathbf{I}$ is an identity matrix, $\mathbf{y}_{t}\in \mathbb{R}^{t-1}$ contains the history of $y_{\tau}$, $\mathbf{k}_{t}(x,a)\in\mathbb{R}^{t-1}$
contains  $k([x,a],[x_{\tau},a_{\tau}])$,
and $\mathbf{K}_{t}\in \mathbb{R}^{(t-1)\times (t-1)}$ contains
$ k([x_{\tau},a_\tau],[x_{\tau'},a_{\tau'}])$, for $\tau, \!\tau'\!\!\in \!\{1,\cdots,t-1\}$.

The confidence width of kernel ridge regression is given in the following concentration lemma followed by a definition of reward UCB.
\begin{lemma}[Concentration of Kernel Ridge Regression]\label{lma:KernelConcent}
	Assume that the reward function $f(x,a)$ satisfies $\left| f(x,a)\right|\leq B$, the noise $n_t$ satisfies a sub-Gaussian distribution with parameter $b$,
	and  kernel ridge regression is used to estimate the reward function. With a probability of at least $1-\delta, \delta\in (0,1)$, for all $a\in \mathcal{A}$ and $t\in \mathbb{N}$, the estimation error satisfies
	$
	|\hat{f}_{t}(x,a)-f(x,a)|\leq h_ts_{t}(x,a),
	$
	where $h_t=\sqrt{\lambda}B+b\sqrt{\gamma_{t}-2\log\left( \delta\right) }$,   $\gamma_{t}=\log \det(\mathbf{I}+\mathbf{K}_{t}/\lambda)\leq \bar{d}\log (1+\frac{t}{\bar{d}\lambda} )$ and $\bar{d}$ is the rank of $\mathbf{K}_{t}$. Let $\mathbf{V}_t=\lambda \mathbf{I}+\sum_{s=1}^{t}\phi(x,a)\phi(x,a)^\top$, the squared confidence width is given by $ s^2_{t}(x,a)=\phi(x,a)^\top\mathbf{V}^{-1}_{t-1}\phi(x,a)=\frac{1}{\lambda}k([x,a],[x,a])-\frac{1}{\lambda}\mathbf{k}_{t}(x,a)^T(\mathbf{K}_{t}+\lambda \mathbf{I})^{-1}\mathbf{k}_{t}(x,a)$.
\end{lemma}
\begin{definition}\label{def:defucb}
	Given arm $a\in\mathcal{A}$ and context $x\in\mathcal{X}$,
	the reward UCB (Upper Confidence Bound)  is defined as
	$
	U_t\left(x,a \right) =\hat{f}_{t}\left( x,a\right) +h_t s_{t}(x,a).
	$
\end{definition}

The next lemma shows that the term $s_t\left(x_{t},a_t \right)$
has a vanishing impact on regret over time.
\begin{lemma}\label{lma:sumvarbatch}
	The sum of confidence widths given 
	$x_t$ for $t\in\{1,\cdots, T\}$ satisfies $\sum_{t=1}^Ts^2_{t}(x_t,a_t)\!\leq 2\gamma_{T}$, where $\gamma_{T}=\log \det(\mathbf{I}+\mathbf{K}_{T}/\lambda)\leq \bar{d}\log (1+\frac{T}{\bar{d}\lambda} )$
	and $\bar{d}$ is the rank of $\mathbf{K}_{T}$. 
\end{lemma}

Then, we give the definition of \emph{UCB-optimal} arm which is important in our algorithm designs.
\begin{definition}\label{def:ucboptarm}
	Given context ${x}\in\mathcal{X}$,
	the \emph{UCB-optimal} arm  is defined as
	$
	A^{\dagger}_t\left(x \right)=\arg\max_{a\in{\mathcal{A}}}U_t\left(x,a \right).
	$
\end{definition}
Note that if the received contexts are perfect, i.e. $\hat{x}_t=x_t$, the standard contextual UCB strategy selects arm at round $t$ as $A^{\dagger}_t\left(x_t \right)$.
Under the cases with imperfect context, a naive policy (which we call \ouralg) is simply oblivious of context errors, i.e. the agent selects the UCB-optimal arm regarding imperfect context $\hat{x}_t$, denoted as $a^{\dagger}_t=A^{\dagger}_t\left(\hat{x}_t\right)$,
by simply viewing the imperfect context $\hat{x}_t$
as true context. Nonetheless, if we want to guarantee the arm selection performance even in the worst case, robust arm selection that accounts for context errors is needed. 

\begin{figure*}[!t]	
	\centering
	\subfigure[Type-I Robustness]{
		\label{fig:robust1illustration} \includegraphics[width={0.47\textwidth}]{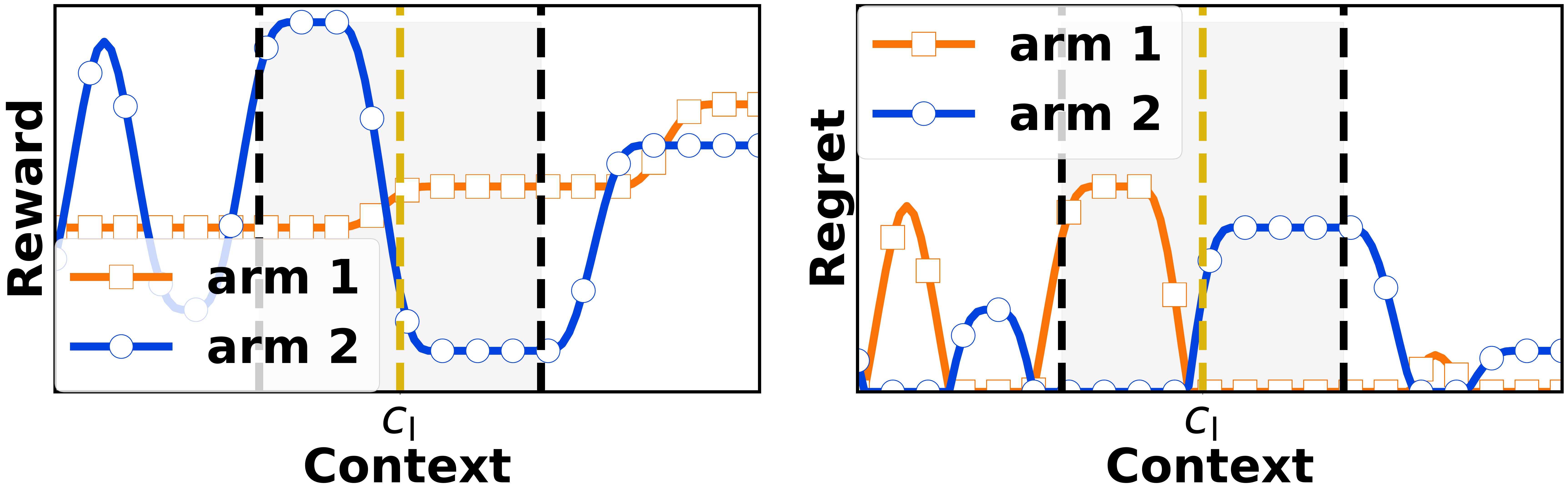}
	}\hspace{0.03\textwidth}
	\subfigure[Type-II Robustness]{
		\label{fig:robust2illustration}
		\includegraphics[width=0.47\textwidth]{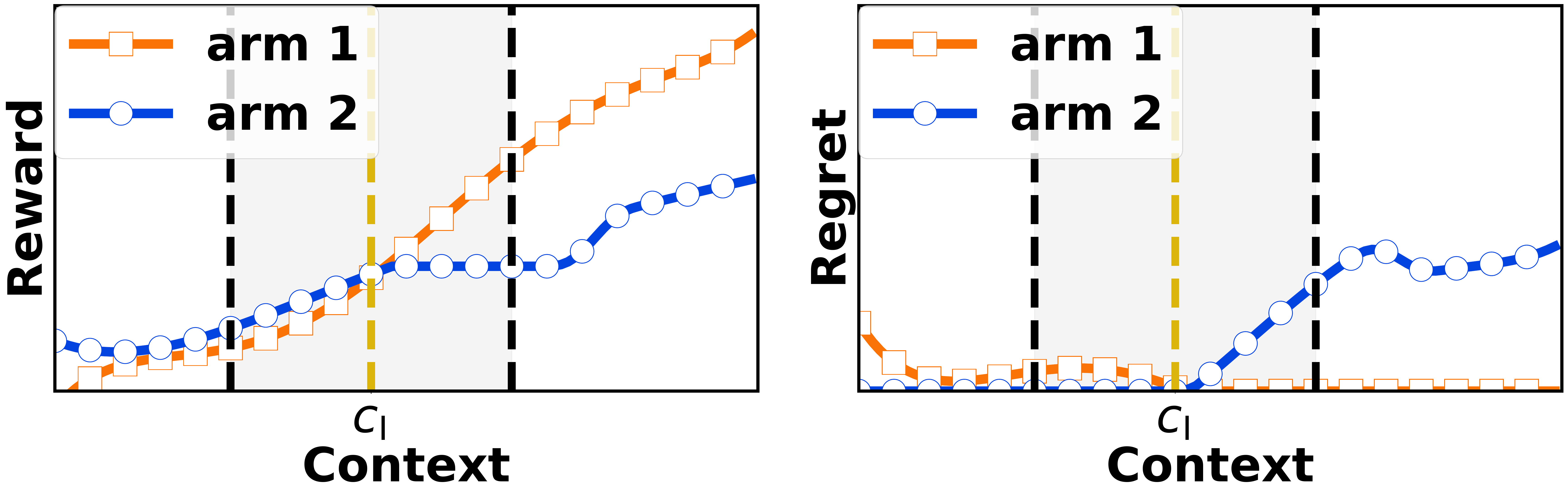}
	}%
	\centering
	\caption{Illustration of reward and regret functions that Type-I and Type-II robustness objectives are suitable for, respectively.
		The golden dotted vertical line represents the imperfect context $c_\mathrm{I}$, and the gray region represents the defense region $\mathcal{B}_{\Delta}(c_\mathrm{I})$.}
	\label{fig:regretrewardrobustness}
\end{figure*}

\section{Robustness  Objectives}
In the existing bandit literature such as \cite{Bandit_Adversarial_Auer_COLT_2016_pmlr-v49-auer16,Sequential_Batch_bandit_han2020sequential,Bandit_LinUCB_WWW_2010_Li:2010:CAP:1772690.1772758},
maximizing the cumulative reward is equivalent to minimizing
the cumulative regret, under the assumption
of perfect context for arm selection.
In this section, we will show that maximizing the worst-case reward is equivalent to minimizing a \emph{pseudo} regret and is different from minimizing the worst-case true regret. 

\subsection{Type-I Robustness}
With imperfect context, one approach to robust arm selection is to maximize the worst-case reward.
With perfect knowledge of reward function, the oracle arm that maximizes the worst-case reward at round $t$ is
\begin{equation}\label{eqn:maxminreward}
\bar{a}_t=\arg\max_{a\in \mathcal{A}}\min_{{x}\in\mathcal{B}_{\Delta}(\hat{x}_t) }f(x,a).
\end{equation}
For analysis in the following sections, given
$\bar{a}_t$, the corresponding context for the worst-case reward  is denoted as
\begin{equation} \bar{{x}}_t=\arg\min_{x \in \mathcal{B}_{\Delta}(\hat{x}_t)}f\left(x,\bar{a}_t \right),
\end{equation}
and the resulting optimal
worst-case reward is denoted as
\begin{equation}\label{eqn:worstreward}
MF_t=f\left(\bar{x}_t,\bar{a}_t \right).
\end{equation}

Next, Type-I robustness objective is defined based on the difference
$\sum_{t=1}^TMF_t-F_T$, where
$F_T=\sum_{t=1}^Tf(x_{t},a_t)$ is the actual cumulative reward.

\begin{definition}\label{def:type-I}
	If, with an arm selection strategy $\left\lbrace a_1, \cdots, a_T\right\rbrace $, the difference between the optimal cumulative worst-case reward and the cumulative true reward $\sum_{t=1}^TMF_t-F_T$ is sub-linear with respect to $T$, then the strategy achieves Type-I robustness.
\end{definition}
If an arm selection strategy achieves Type-I robustness, the lower bound for the true reward $f\left(x_t,a_t \right)$ approaches the optimal worst-case reward $MF_t$ in the defense region as $t$ increases.
Therefore, a strategy achieving type-I robustness objective  can prevent very low reward.  For example, in Fig.~\ref{fig:robust1illustration}, arm 1 is the one that maximizes the worst-case reward, which is not necessarily optimal but always avoids extremely low reward under any context in the defense region.

Note that maximizing the worst-case reward is equivalent to minimizing the robust regret defined in \cite{Bandit_RobustDistributionalContextual_AISTATS_2020_kirschner2020distributionally}, which is written using our formulation as \begin{equation}\label{eqn:pseudoregret}
\bar{R}_T=\sum_{t=1}^T\min_{x\in\mathcal{B}_{\Delta}(\hat{x}_t)}f\left(x,\bar{a}_t \right) \!-\!\min_{x\in\mathcal{B}_{\Delta}(\hat{x}_t)}f\left(x,a_t \right).
\end{equation}
However, this robust regret is a \emph{pseudo} regret because the rewards of oracle arm $\bar{a}_t$ and selected arm $a_t$ are compared under different contexts
(i.e., their
respective worst-case contexts), and it is not an upper or lower bound of the true regret $R_T$. To obtain a robust regret performance, we need to define another robustness objective based on the true regret.

\subsection{Type-II Robustness}
To provide robustness for the regret with imperfect context, we can minimize the cumulative worst-case regret, which is expressed as
\begin{equation}\label{eqn:worstregret}
\tilde{R}_T=\sum_{t=1}^T\max_{x\in\mathcal{B}_{\Delta}(\hat{x}_t)}\left[ f\left(x,A^*(x)\right) \!-\!f\left(x,a_t \right)\right].
\end{equation}
Clearly, the true regret $R_T\leq \tilde{R}_T$,
and minimizing the worst-case regret is equivalent to minimizing
an upper bound for the true regret. Define the instantaneous regret function with respect to context $x$ and arm $a$ as $
r\left(x,a \right) =f(x,A^{*}\left(x \right) )-f(x,a)
$. Since given the reward function the optimization is decoupled among different rounds, the robust oracle arm to minimize the worst-case regret at round $t$ is
\begin{equation}\label{eqn:minmaxregret}
\tilde{a}_t=\arg\min_{a \in \mathcal{A}}\max_{x \in \mathcal{B}_t(\hat{x}_t)}r(x,a).
\end{equation}
For analysis in the following sections, given $ \tilde{a}_t$, the corresponding context for the worst-case regret is denoted as
\begin{equation} \label{eqn:solutionmmreg}
\tilde{x}_t=\arg\max_{x \in \mathcal{B}_{\Delta}(\hat{x}_t)}r(x,\tilde{a}_t),
\end{equation}
and the resulting optimal worst-case regret is
\begin{equation}\label{eqn:optworstreg}
MR_t=r(\tilde{x}_t,\tilde{a}_t).
\end{equation}
Now, we can give the definition of Type-II robustness as follows.
\begin{definition}\label{def:type-II}
	If, with an arm selection strategy $\left\lbrace a_1, \cdots, a_T\right\rbrace $, the difference between the cumulative true regret and the optimal cumulative worst-case regret $R_T-\sum_{t=1}^{T}MR_t$ is sub-linear with respect to $T$, then the strategy achieves Type-II robustness.
\end{definition}

If an arm selection strategy achieves Type-II robustness, as time increases, the upper bound for the true regret $r(x_t,a_t)$ also approaches the optimal worst-case regret $MR_t$.
Hence, a strategy achieving type-II robustness objective can prevent a high regret.
As shown in Fig. \ref{fig:robust2illustration}, arm 1 is selected by minimizing the worst-case regret, which is a robust arm selection because the regret of arm 1 under any context in the defense region is not too high.

\subsection{Comparison of  Two Robustness Objectives}
The two types of robustness correspond to the algorithms maximizing the worst-case reward and minimizing the worst-case regret, respectively. In many cases, they result in different arm selections. Take the two scenarios in Fig.~\ref{fig:regretrewardrobustness} as examples. In the scenario of Fig.~\ref{fig:robust1illustration}, given the defense region, arm 1 is selected by maximizing the worst-case reward and arm 2 is selected by minimizing the worst-case regret.  It can be observed that the worst-case regrets of the two arms are very close, but the worst-case reward of arm 2 is much lower than that of arm 1. Thus, the strategy of maximizing the worst-case reward is more suitable for this scenario. Differently, in the scenario of Fig.~\ref{fig:robust2illustration}, arm 2 is selected by maximizing the worst-case reward and arm 1 is selected by minimizing the worst-case regret. Since the worst-case rewards of the two arms are very close and the worst-case regret of arm 2 is much larger than arm 1, it is more suitable to minimize the worst-case regret.

\section{Robust Bandit Arm Selection}\label{sec:robustUCB}
In this section,
we propose two robust arm selection algorithms: (1)
\ouralgtwo (Maximize Minimum Upper Confidence Bound), which aims to maximize the minimum reward (Type-I robustness objective); and (2) \ouralgthree (Minimize Worst-case  Degradation), which aims to minimize the maximum regret (Type-II robustness objective). We derive the regret and reward bounds for both algorithms and the proofs are available in \cite{yang2021robust}.

\subsection{\ouralgtwo: Maximize Minimum UCB}\label{sec:maxminUCB}

\subsubsection{Algorithm}

\begin{algorithm}[!t]
	\caption{Robust Arm Selection with Imperfect Context}\label{alg:robust_arm}
	\begin{algorithmic}
		\REQUIRE  Context error budget $\Delta$
		\FOR {$t=1,\cdots,T$}
		\STATE   Receive imperfect context $\hat{x}_t$.
		\STATE   Select arm $a^{\textrm{I}}_t$ to solve Eqn.~\eqref{equ:minucb}
		in \ouralgtwo; or select arm $a^{\textrm{II}}_t$ to solve Eqn.~\eqref{equ:maxminsubopt} in \ouralgthree
		\STATE  Observe the true context $x_t$ and the reward $y_t$.
		\ENDFOR
	\end{algorithmic}
\end{algorithm}

To achieve type-I robustness,
\ouralgtwo in Algorithm~\ref{alg:robust_arm} selects an arm $a^{\textrm{I}}_t$ by maximizing the minimum UCB within the defense region  $\mathcal{B}_{\Delta}(\hat{x}_t)$:
\begin{equation}\label{equ:minucb}
a^{\textrm{I}}_t =\arg\max_{a\in\mathcal{A}}\min_{{x} \in \mathcal{B}_{\Delta}(\hat{x}_t)}U_t\left(x,a \right).
\end{equation}
The corresponding context that attains the minimum UCB in Eqn.\eqref{equ:minucb} is $x^{\textrm{I}}_t=\min_{x \in\mathcal{B}_{\Delta}(\hat{x}_t)}U_t\left(x,a^{\textrm{I}}_t \right)$.

\subsubsection{Analysis}\label{analysismaxminucb}

The next theorem gives a lower bound of the cumulative true reward of \ouralgtwo in
terms of the optimal worst-case reward and a sub-linear term.
\begin{theorem}\label{the:rewardmmucb}
	If \ouralgtwo is used to select arms with imperfect context, then for any true contexts $x_t\in \mathcal{B}_{\Delta}(\hat{x}_t)$ at round $t, t=1,\cdots, T$, with a probability of $1-\delta,\delta\in (0,1)$, we have the following lower bound on the worst-case cumulative reward
	\begin{equation}
	F_T\geq \sum_{t=1}^T MF_t-2h_T\sqrt{2T\bar{d}\log (1+\frac{T}{\bar{d}\lambda} )}
	\end{equation}
	where $MF_t$ is the optimal worst-case reward in Eqn.~\eqref{eqn:worstreward}, $\bar{d}$ is the rank of $\mathbf{K}_t$ and $h_T$ is given in Lemma \ref{lma:KernelConcent}.
\end{theorem}
\begin{remark}
	Theorem \ref{the:rewardmmucb} shows that by \ouralgtwo, the difference between the optimal cumulative worst-case reward and the cumulative true reward is sub-linear and thus effectively achieves Type-I robustness according to Definition \ref{def:type-I}.
	This means that the reward by \ouralgtwo has a bounded sub-linear
	gap compared to the optimal worst-case reward
	$\sum_{t=1}^T MF_t$ obtained with perfect knowledge
	of the reward function.
	\hfill $\square$
\end{remark}

We are also interested in the cumulative true regret of \ouralgtwo which is given in the following corollary.
\begin{corollary}\label{the:regretmmucb}
	If \ouralgtwo is used to select arms with imperfect context, then for any true contexts $x_t\in  \mathcal{B}_{\Delta}(\hat{x}_t)$ at round $t, t=1,\cdots, T$, with a probability of $1-\delta,\delta\in (0,1)$, we have the following bound on the cumulative true regret defined in Eqn.~\eqref{eqn:trueregret}:
	\begin{equation}\label{equ:regretmmucbwc}
	\begin{split}
	R_T \!\leq\! \sum_{t=1}^T\overline{MR}_t +2h_T\sqrt{2T\bar{d}\log (1+\frac{T}{\bar{d}\lambda} )}
	\end{split}
	\end{equation}
	where $\overline{MR}_t=\max_{x \in \mathcal{B}_{\Delta}(\hat{x}_t)} f\left(x,A^*\left( x \right)  \right) -MF_t$, $MF_t$ is the optimal worst-case reward in Eqn.~\eqref{eqn:worstreward} .
\end{corollary}
\begin{remark}
	Corollary \ref{the:regretmmucb} shows that the worst-case regret by \ouralgtwo
	can be quite larger than the optimal worst-case regret $MR_t$ given in Eqn.~\eqref{eqn:optworstreg} (Type-II robustness objective). Actually, despite
	being robust in terms of rewards, arms selected
	by \ouralgtwo can still have very large regret as shown in Fig.~\ref{fig:robust2illustration}. Thus, to achieve type-II robustness, it is necessary to develop an arm selection algorithm that minimizes the worst-case regret.
\end{remark}

\subsection{\ouralgthree: Minimize Worst-case Degradation}\label{sec:minWD}

\subsubsection{Algorithm}
\ouralgthree is designed to asymptotically minimize the worst-case regret. Without the oracle knowledge of reward function, \ouralgthree performs arm selection based on the upper bound of regret.
Denote $D_{a}\left(x \right) = U_t\left(x, A^{\dagger}_t\left(x \right)   \right) - U_t\left(x,a \right)$ referred to as UCB \emph{degradation} at context $x$.
By Lemma \ref{lma:KernelConcent}, the instantaneous true regret can be bounded as
\begin{equation}\label{eqn:boundregretbysubopt}
\begin{split}
r\!\!\left(x_t,a_t \right)&\!\leq\!  \left[D_{a_t}\left(x_t \right) \!+\!2h_ts_t\left(x_t,a_t \right) \right]\\
&\!\leq\! \overline{D}_{a_t}\!+\!2h_ts_t\left(x_t,a_t \right),
\end{split}
\end{equation}
where $\overline{D}_{a_t}=\max_{x\in\mathcal{B}_{\Delta}(\hat{x}_t)}D_{a}\left(x\right)$ is called the \emph{worst case degradation}, and  $2h_ts_t\left(x_t,a_t \right) $ has a vanishing impact by Lemma \ref{lma:sumvarbatch}. Thus, 
to minimize worst-case
regret, 
\ouralgthree minimizes its upper bound $\overline{D}_{a_t}$ excluding
the vanishing term $2h_ts_t\left(x_t,a_t \right) $, i.e.
\begin{equation}\label{equ:maxminsubopt}
a^{\textrm{II}}_t=\min_{a\in\mathcal{A}}\max_{\mathbf{x} \in \mathcal{B}_{\Delta}(\hat{x}_t)}\left\lbrace  U_t\left(x, A^{\dagger}_t\left(x \right)   \right) - U_t\left(x,a \right) \right\rbrace.
\end{equation}
The context that attains the worst case in Eqn.~\eqref{equ:maxminsubopt} is written as $x^{\textrm{II}}_t=\arg\max_{x \in \mathcal{B}_{\Delta}(\hat{x}_t)}D_{a^{\textrm{II}}_t}\left(x\right)$.

\subsubsection{Analysis}
Given arm $a^{\textrm{II}}_t$ selected by \ouralgthree, the next lemma gives an upper bound of worst-case degradation.
\begin{lemma}\label{lma:boudnsubopt}
	If \ouralgthree is used to select arms with imperfect context, then for each $t=1,2,\cdots, T$, with a probability at least $1-\delta,\delta\in (0,1)$, we have
	\begin{equation}
	\overline{D}_{a^{\mathrm{II}}_t,t}\leq MR_t+2h_ts_t\left(\dot{x}_t,  A^{\dagger}_t\left(\dot{x}_t \right) \right),
	\end{equation}
	where $MR_t$ is the optimal worst-case regret defined in Eqn.~\eqref{eqn:optworstreg}, $\dot{x}_t=\arg\max_{x\in\mathcal{B}_{\Delta}(\hat{x}_t)}D_{\tilde{a}_t}\left(x \right)$ is the context that maximizes the degradation given the arm $\tilde{a}_t$ defined for the optimal worst-case regret in Eqn.~\eqref{eqn:solutionmmreg}.
\end{lemma}
Then, in order to show that $\overline{D}_{a^{\textrm{II}}_t,t}$ approaches $MR_t$, we need to prove that $2h_ts_t\left(\dot{x}_t,  A^{\dagger}_t\left(\dot{x}_t \right)\right)$ vanishes
as $t$ increases. But, this is difficult because the considered sequence $\left\lbrace \dot{x}_t,  A^{\dagger}_t\left(\dot{x}_t \right)\right\rbrace $ is different from the actual sequence of context and selected arms $\left\lbrace x_t, a_t^{\textrm{II}}\right\rbrace$ under \ouralgthree.
To circumvent this issue,
we first introduce the concept of $\epsilon-$ covering \cite{math_epsilon_covering}. Denote $\Phi=\mathcal{X}\times\mathcal{A}$ as the context-arm space. If a finite set $\Phi_{\epsilon}$ is an $\epsilon-$ covering of the space $\Phi$, then for each $\varphi\in\Phi$, there exists at least one $\bar{\varphi}\in\Phi_{\epsilon}$ satisfying $\|\varphi-\bar{\varphi} \|_2\leq \epsilon$. Denote $\mathcal{C_\epsilon}\left(\bar{\varphi} \right) =\left\lbrace \varphi\mid\|\varphi-\bar{\varphi} \|_2\leq \epsilon\right\rbrace $ as the cell with respect to $\bar{\varphi}\in\Phi_{\epsilon}$. Since the dimension of the entries in $\Phi$ is $d$, the size of the $\Phi_{\epsilon}$ is $|\Phi_{\epsilon}|\sim O\left( \frac{1}{\epsilon^d}\right) $. Besides, we assume the mapping function $\phi$ is Lipschitz continuous, i.e. $\forall x, y\in \Phi$, $\|\phi(x)-\phi(y)\|\leq L_{\phi}\|x-y\|$.
Next, we prove the following proposition to bound the sum of confidence widths under some conditions.

\begin{proposition}\label{prop:sumconfwd}
	Let $\mathcal{X}_T=\left\lbrace x_{a_1,1},\cdots,  x_{a_T,T}\right\rbrace$ be the sequence of true contexts and selected arms by bandit algorithms and $\dot{\mathcal{X}}_T=\left\lbrace  \dot{x}_{\dot{a}_1,1},\cdots,  \dot{x}_{\dot{a}_T,T} \right\rbrace $ be the considered sequence of contexts and actions. Suppose that both $x_{a_t,t}$  and $\dot{x}_{\dot{a}_t,t}$belong to $\Phi$. Besides, with an $\epsilon-$ covering $\Phi_{\epsilon}\subseteq \Phi$, $\epsilon> 0$,  there exists $ \kappa \geq 0$ such that two conditions are satisfied: First, $\forall \bar{\varphi}\in \Phi_{\epsilon}$, $\exists t\leq \left \lceil  \kappa/\epsilon^d\right \rceil$ such that $x_{a_t,t}\in \mathcal{C}_{\epsilon}\left(\bar{\varphi} \right) $. Second, if at round $t$, $x_{a_t,t}\in \mathcal{C}_{\epsilon}\left(\bar{\varphi} \right) $ for some $\bar{\varphi}\in \Phi_{\epsilon}$, then $\exists  t\leq t'< t+\left \lceil  \kappa/\epsilon^d\right \rceil$ such that $x_{a_t',t'}\in \mathcal{C}_{\epsilon}\left(\bar{\varphi} \right) $. If the mapping function $\phi$ is Lipschitz continuous with constant $L_{\phi}$, the sum of squared confidence widths is bounded as
	\begin{equation*}
	\begin{split}
	\sum_{t=1}^T \!s^2_{t}\left(\dot{x}_{\dot{a}_t,t}\right)
	\!\!\leq \!\!\sqrt{T}\!\left(\! 4\tilde{d}\log\left(\!1\!+\!\frac{T}{\tilde{d}\lambda} \!\right)\!+\!\frac{1}{\lambda}\!	\right) \!\!+\!\!\frac{\!8L^2_{\phi}\kappa^{2/d}\!}{\lambda}T^{1-1/d},
	\end{split}
	\end{equation*}
	where $d$ is the dimension of $x_{a_t,t}$, $\tilde{d}$ is the effective dimension defined in the proof, $s^2_{t}\left(\dot{x}_{\dot{a}_t,t}\right)\!=\!\phi(\dot{x}_{\dot{a}_t,t})^{\top}\mathbf{V}^{-1}_{t-1}\phi(\dot{x}_{\dot{a}_t,t})$ and $\mathbf{V}_t\!=\!\lambda\mathbf{I}+\sum_{s=1}^t\phi(x_{a_s,s})\phi(x_{a_s,s})^{\top}$.
\end{proposition}
\begin{remark}
	The conditions in Proposition \ref{prop:sumconfwd} guarantee that the
	time interval between the events that true context-arm feature lies in the same cell is not larger than $\left \lceil  \kappa/\epsilon^d\right \rceil$, which is proportional to the size of the $\epsilon$-covering $|\Phi_{\epsilon}|$.
	That means,
	similar contexts and selected arms occur in the true sequence repeatedly if $T$ is large enough. If contexts are sampled from a bounded space $\mathcal{X}$ with some distribution, then similar contexts will occur repeatedly.
	Also, note that the arm in our considered sequence $A^{\dagger}_t\left( \dot{x}_t\right)$ is the UCB-optimal arm, which becomes close to the optimal arm for $\dot{x}_t$ if the confidence width is sufficiently small.
	Hence,
	there exists some context error budget
	sequence $\left\lbrace \Delta_t\right\rbrace $ such that, starting from a certain round $T_0$, the two conditions are satisfied. The two conditions in Proposition \ref{prop:sumconfwd} are
	mainly for theoretical analysis of \ouralgthree. 
\end{remark}

By Lemma \ref{lma:boudnsubopt} and Proposition \ref{prop:sumconfwd},
we bound the cumulative regret of \ouralgthree. 

\begin{theorem}\label{the:regretminwd}
	If \ouralgthree is used to select arms with imperfect context and as time goes on, and the conditions in Proposition \ref{prop:sumconfwd} are satisfied, then for any true context $x_t\in \mathcal{B}_{\Delta}(\hat{x}_t)$ at round $t, t=1,\cdots, T$,  with a probability of $1-\delta,\delta\in (0,1)$, we have the following bound on the cumulative true regret:
	\begin{equation*}\label{equ:regretminwd}
	\begin{split}
	R_T \leq &\sum_{t}^{T}MR_t+2h_TT^{\frac{3}{4}}\sqrt{\left( 4\tilde{d}\log\left(1+\frac{T}{\tilde{d}\lambda} \right)	+\frac{1}{\lambda} 	\right)} +\\
	&4\sqrt{\frac{2}{\lambda}}L_{\phi}\kappa^{\frac{1}{d}}h_TT^{1-\frac{1}{2d}}+2h_T\!\!\sqrt{2T\bar{d}\log (1\!+\!\frac{T}{\bar{d}\lambda} )},
	\end{split}
	\end{equation*}
	where $MR_t$ is the optimal worst-case regret for round $t$ in Eqn.~\eqref{eqn:optworstreg}, $d$ is the dimension of $x_{a_t,t}$, $\tilde{d}$ is the effective dimension defined in the proof of Proposition \ref{prop:sumconfwd}, $\bar{d}$ is the rank of $\mathbf{K}_t$ and $h_T$ is given in Lemma \ref{lma:KernelConcent}.
\end{theorem}

\begin{figure*}[!t]	
	\centering
	\subfigure[Robust regret.]{
		\label{fig:regret_pseudo}
		\includegraphics[width=0.32\textwidth]{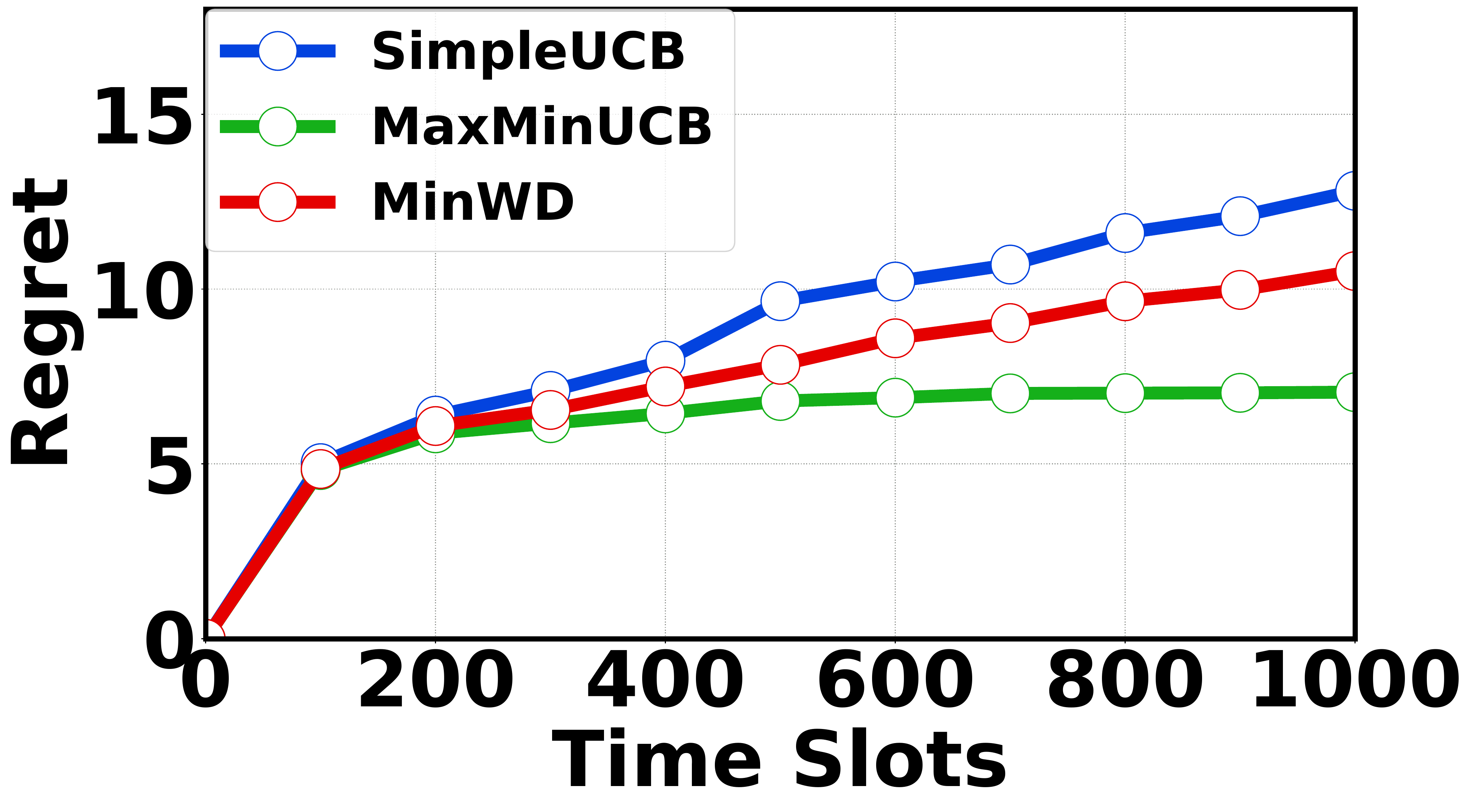}
	}%
	\subfigure[Worst-case regret.]{
		\label{fig:regret_worst}
		\includegraphics[width=0.32\textwidth]{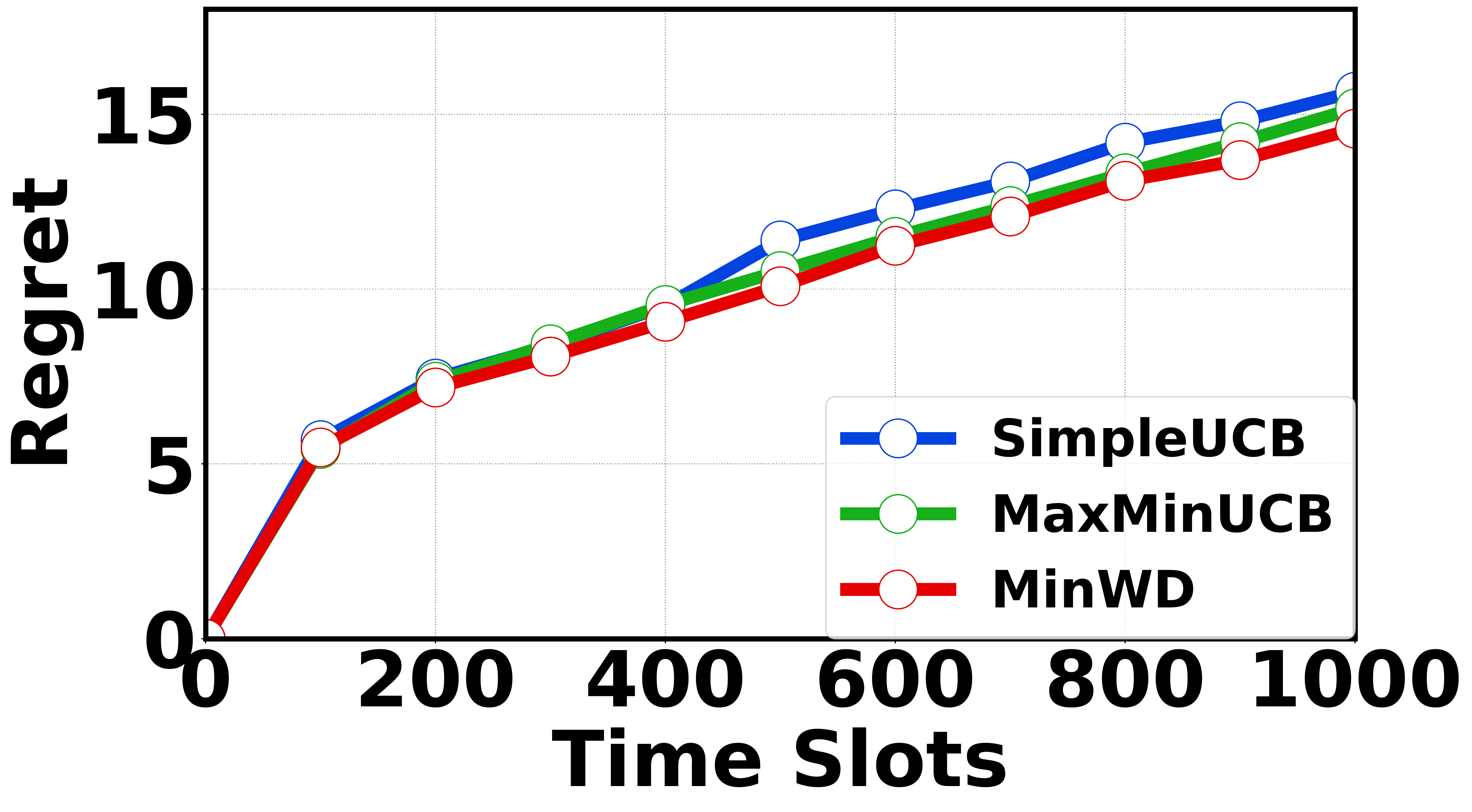}
	}%
	\subfigure[True regret.]{
	\label{fig:regret_true} \includegraphics[width={0.32\textwidth}]{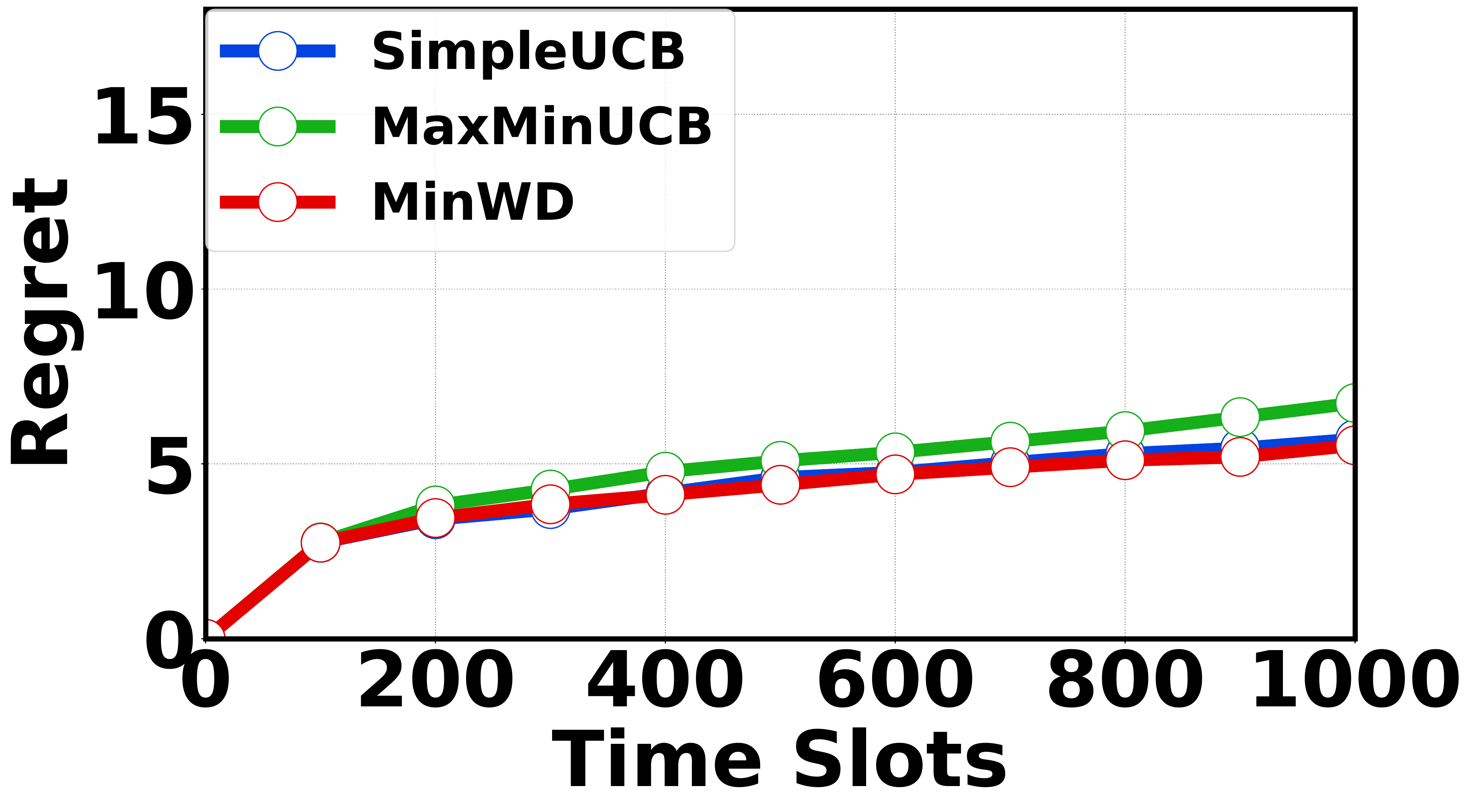}
}
	\centering
	\caption{Different cumulative regret objectives for different algorithms.}
	\label{fig:simulationregret}
\end{figure*}

\begin{remark}
	Theorem~\ref{the:regretminwd} shows that by \ouralgthree, $R_T-\sum_{t=1}^TMR_t$ is sub-linear w.r.t. $T$ and thus Type-II robustness is effectively achieved according to Definition \ref{def:type-II}. This means the true regret bound
	approaches $\sum_{t}^{T}MR_t$, the optimal worst-case regret, asymptotically.
\end{remark}

Next, in parallel with \ouralgtwo,
we derive the bound of true reward for \ouralgthree.
\begin{corollary}\label{cor:rewardminwd}
	If \ouralgthree is used to select arms with imperfect context and as time goes on, and the true sequence of context and arm obeys the conditions in Proposition \ref{prop:sumconfwd}, then for any true contexts $x_t\in \mathcal{B}_{\Delta}(\hat{x}_t)$ at round $t, t=1,\cdots, T$, with a probability of $1-\delta,\delta\in (0,1)$, 	we have the following lower bound of the cumulative reward
	\begin{equation*}\label{equ:rewardminwd}
	\begin{split}
	F_T &\geq \!\sum_{t=1}^{T}\left[ MF_t\!-\!MR_t\right]\!-\!2h_TT^{\frac{3}{4}}\!\sqrt{\!\left( 4\tilde{d}\log\left(1+\frac{T}{\tilde{d}\lambda} \right)\!\!	+\!\!\frac{1}{\lambda} 	\right)}\\
	&-4\sqrt{\frac{2}{\lambda}}L_{\phi}\kappa^{\frac{1}{d}}h_TT^{1-\frac{1}{2d}}-2h_T\!\!\sqrt{2T\bar{d}\log (1\!+\!\frac{T}{\bar{d}\lambda} )},
	\end{split}
	\end{equation*}
	where $MR_t$ is the optimal worst-case regret for round $t$ in Eqn.~\eqref{eqn:optworstreg}, $d$ is the dimension of $x_{a_t,t}$, $\tilde{d}$ is the effective dimension defined in the proof of Proposition \ref{prop:sumconfwd}, $\bar{d}$ is the rank of $\mathbf{K}_t$, and $h_T$ is given in Lemma \ref{lma:KernelConcent}.
\end{corollary}
\begin{remark}
	Corollary \ref{cor:rewardminwd} shows that as $t$ becomes sufficiently large, the difference between the optimal worst-case reward and the true reward of the selected arm is no larger than the optimal worst-case regret $MR_t$. With perfect context, we have $MR_t=0$, and hence \ouralgtwo and \ouralgthree both asymptotically maximize the reward, implying that these two types of robustness are the same under perfect context.
\end{remark}
\subsection{Summary of Main Results}
\begin{table}[!t]
	\caption{Summary of Analysis}
	\label{tab:regretanalysis}
	\begin{center}
		\begin{tabular} { | p{2.1cm} | p{2.5cm}| p{2.5cm} | }
			\hline
			\multicolumn{1}{|c|}{\textbf{Algorithms }}&\textbf{Regret}&\textbf{Reward}\\
			\hline
			\multicolumn{1}{|c|}{\ouralgtwo}& $\sum_{t=1}^T\!\overline{MR}_t\!+O(\!\sqrt{T\!\!\log \!T}\!)$& $\sum_{t=1}^T\!\!MF_t-O(\!\sqrt{T\!\!\log \!T}\!)$\\
			\hline
			\multicolumn{1}{|c|}{\ouralgthree}&$\sum_{t=1}^T\!\!MR_t\!+\!O(T^{\frac{3}{4}}\!\!\sqrt{\log \!T}\!+\!T^{1-\frac{1}{2d}}\!+\!\sqrt{T\!\!\log \!T}\!)$ &$\sum_{t}^{T}\!\!\left[ MF_t\!-\!MR_t\right]\!-\!O(T^{\frac{3}{4}}\!\sqrt{\log T}\!+\!T^{1-\frac{1}{2d}}+\!\sqrt{T\!\!\log \!T}\!)$ \\
			\hline
		\end{tabular}
	\end{center}
\end{table}

We summarize our 
analysis
of \ouralgtwo and \ouralgthree in Table~\ref{tab:regretanalysis},
while the algorithms details are available in Algorithm \ref{alg:robust_arm}. In the table, $d$ is the dimension of context-arm vector $[x,a]$,
$\overline{MR}_t=\max_{x \in \mathcal{B}_{\Delta}(\hat{x}_t)} f\left(x,A^*\left( x \right)  \right) -MF_t$, and $MF_t$ and $MR_t$ are defined in Eqn.~\eqref{eqn:worstreward}
and~\eqref{eqn:optworstreg}, respectively.  Type-I and type-II robustness objectives are achieved by \ouralgtwo and \ouralgthree respectively.

\section{Simulation}\label{sec:simulation}

Edge computing is a promising technique to meet the demand of latency-sensitive applications 
\cite{Edge_Survey_Weisong_Shi_IoTJournal_2016_7488250}.
Given
multiple heterogeneous edge datacenters
located in different locations, which one should be selected?
Specifically, each edge
datacenter is viewed as an arm, and the users' workload is context that can only be predicted prior to arm selection.
Our goal is to learn datacenter selection
to optimize the latency in a robust manner
given imperfect workload  information.
We assume that the service rate of the edge datacenter $a$, $a \in\mathcal{A}$, is $\mu_a$, the computation latency  satisfies an M/M/1 queueing model and the average communication delay between this datacenter
and users is $p_a$.
Hence, the average total latency cost can be expressed as
$
l(x,a)=p_a\cdot x+\frac{x}{\mu_a-x}
$
which is commonly-considered in the literature \cite{LinWiermanAndrewThereska,Shaolei_EdgeComputing_Online_Learning_AutoScaling_Renewable_JieXu_2018,LinLiuWiermanAndrew_IGCC_2012}.
The detailed settings are given in \cite{yang2021robust}.

In Fig.~\ref{fig:simulationregret}, we compare different algorithms in terms of three cumulative regret objectives: robust regret in Eqn.~\eqref{eqn:pseudoregret}, worst-case regret in Eqn.~\eqref{eqn:worstregret} and true regret in Eqn.~\eqref{eqn:trueregret}.
We consider the following algorithms:
\ouralg with imperfect context, \ouralgtwo with imperfect context and \ouralgthree with imperfect context. Given a sequence of true contexts, imperfect context sequence is generated by sampling i.i.d. uniform distribution over $\mathcal{B}_{\Delta}(x_t)$ at each round.  
In the simulations, Gaussian kernel with parameter $0.1$ is used for reward (loss) estimation.
$\lambda$ in Eqn.~\eqref{equ:estimean} is set as $0.1$. The exploration rate is set as $h_t=0.04$.

As is shown in Fig.~\ref{fig:regret_pseudo}, \ouralgtwo has the best performance of robust regret among the three algorithms. This is because \ouralgtwo targets at type-I robustness objective which is equivalent to minimizing the robust regret. However, \ouralgtwo is not the best algorithm in terms of true regret as is shown in Fig.~\ref{fig:regret_true} since robust regret is not an upper or lower bound of true regret. Another robust algorithm \ouralgthree is also better than \ouralg in terms of robust regret, and it has the best performance among the three algorithms in terms of the worst-case regret, as shown in Fig.~\ref{fig:regret_worst}. This is because the regret of \ouralgthree 
approaches the optimal worst-case regret  (Theorem \ref{the:regretminwd}). \ouralgthree also has a good performance of true regret, which coincides with the fact that the worst-case regret is the upper bound of the true regret.
By comparing the three algorithms in terms of the three regret objectives,
we can clearly see that \ouralgtwo and \ouralgthree achieve
performance robustness in terms of the robust regret
and worst-case regret, respectively.

\section{Conclusion}
In this paper, 
considering a bandit setting with imperfect context,
we propose:  \ouralgtwo which  maximizes the worst-case reward;
and \ouralgthree  which  minimizes the worst-case regret.
Our analysis
of \ouralgtwo and \ouralgthree
based on regret and reward bounds
shows that as time goes on, \ouralgtwo
and \ouralgthree both perform as  asymptotically well
as their counterparts that have perfect knowledge of the reward function.
Finally, we 
consider online edge datacenter selection
and run synthetic simulations for evaluation.

\section*{Acknowledgments}
This work was supported in part by the NSF under grants CNS-1551661 and ECCS-1610471.

\appendix
\section{Applications and Simulation Settings}\label{sec:formulation_motivation}

A key novelty of our work is the consideration
of imperfect context for arm selection, which characterizes many
practical applications such as resource management problems where the true context is difficult to obtain for arm selection until revealed later. We list some examples of these applications in this section and provide the simulation settings.

\subsection{Motivating Applications}

\textbf{Cloud Resource Management.}
Cloud computing platforms are crucial infrastructures offering
utility-style computation resources to users on demand.
To optimize the performance metrics such as latency and cost for these applications, efficient online cloud resource management such as dynamic virtual machine scheduling is necessary.  Contextual bandit learning can be employed in this scenario where the exact workload
information (measured in, e.g., how many requests will arrive per unit time)
is crucial context information, but cannot
be known until the workload actually arrives. Instead,
the agent can only predict the upcoming workload by exploiting
the recent workload history plus other applicable system features.
A real-word example is Amazon's predictive scaling that leverages time series prediction to estimate upcoming workload for virtual machine scheduling \cite{Amazon_AWS_AutoScaling_Documentation}.
In this motivating example, the context prediction error comes primarily from the  workload predictor.

\textbf{Energy Scheduling in Smart Grid.}
Energy load is a crucial context information for energy scheduling in smart grid.
However, a recent study \cite{Baosen_Adversarial_LoadForecasting_SmartGrid_eEnergy_2019_10.1145/3307772.3328314}
shows that the energy load predictor in smart grid can be adversarially attacked to produce load estimates (i.e., context for energy scheduling) with higher-than-normal errors.
Thus, our model also captures a novel adversarial
setting where  erroneously
predicted context is presented to the agent
for arm selection. This example shows
that the use of machine learning-based predictor for acquiring context
to facilitate the agent's arm selection
can also open a new attack surface --- an outside attacker may adversarially modify the predicted context --- which requires a robust algorithm with a provable
worst-case performance guarantee.

\textbf{Online Edge Datacenter Selection.}
With the rapidly increasing number of devices
at the Internet edge,
computational demand  by latency-sensitive applications (e.g.,
assisted driving and virtual reality) has been
escalating. Edge computing is a promising solution to meet the demand, which deploys computation resources at densely-distributed edge datacenters
close to end users and thus reduces the overall latency \cite{Edge_Survey_Weisong_Shi_IoTJournal_2016_7488250}.
Since users' workloads can be processed
in multiple edge datacenters,
dynamic
selection of edge datacenters
plays a key role
for minimizing the overall latency --- given
multiple heterogeneous edge datacenters
located in different locations, which one should be selected?
Here, servers in each edge datacenter
can be either virtual machines rented from third-party service providers
or physical servers owned by the edge computing provider itself.
We refer to this problem as {online edge datacenter selection}.
Our considered bandit
setting applies to the problem of online edge datacenter selection.
Specifically, each edge
datacenter is viewed as an arm, and the users' workload is context that can only be predicted prior to arm selection.
Our goal is to dynamically select edge datacenters to optimize the latency in a robust manner given imperfect information about users' workloads.

\subsection{Simulation Settings in Section \ref{sec:simulation}}
In our simulation, we apply our algorithms to online edge
datacenter selection in the context of edge computing.
We consider users' workloads can be
processed in one of four available edge datacenters, each having different computation capabilities and communication latencies.
Here, we consider a simple latency model to capture first-order effects.
Concretely, we assume that the service rate of the edge datacenter $a$, $a \in\mathcal{A}=\{1,2,3,4\}$, is $\mu_a$, the computation latency  satisfies an M/M/1 queueing model and the average communication delay between this datacenter
and users is $p_a$. The values of $\mu_a$ and $p_a$ are shown
in Table~\ref{tab:setting_simulation}. 
In the simulations, perfect context sequence $\{x_t\}$ is generated by sampling i.i.d. uniform distribution between 10 and 30 for each round, while the defense budget $\Delta$ is set as 2.
The average total latency cost can be expressed as
\begin{equation}\label{equ:latency}
	f(x,a)=\frac{x}{\mu_a-x}+p_a\cdot x
\end{equation}
whose inverse can be equivalently viewed as the reward in our model.
\begin{table}[!h]
	\caption{Simulation Settings}
	\label{tab:setting_simulation}
	\begin{center}
		\begin{tabular} { | p{3cm} | p{0.9cm}| p{0.9cm} |p{0.9cm} |p{0.9cm} | }
			\hline
			\multicolumn{1}{|c|}{\textbf{Datacenter $a$}}&\textbf{I}&\textbf{II}&\textbf{III}&\textbf{IV}\\
			\hline
			\multicolumn{1}{|c|}{${\mu}_a$}& 35& 38&45&51\\
			\hline
			\multicolumn{1}{|c|}{$p_a$}& 0.04& 0.05&0.074&0.088\\
			\hline
		\end{tabular}
	\end{center}
\end{table}

While we use
the simple latency model in Eqn.~\eqref{equ:latency} for generating the
ground-true cost,
the function form and parameters (e.g., $\mu_a$) may
not be known to the agent and needs to be learnt based on latency feedback and revealed context.
Note that some practical factors (e.g., workload parallelism) are beyond the scope of our analysis and incorporating them into our simulation
does not add substantially to our main contribution.

\section{Algorithm and Proofs Related to SimpleUCB}

\subsection{Algorithm of \ouralg}
We describe  \ouralg in Algorithm~\ref{alg:standardUCB}.

\begin{algorithm}[!h]
	\caption{Simple UCB (\ouralg)}\label{alg:standardUCB}
	\begin{algorithmic}
		\FOR {$t=1,\cdots,T$}
		\STATE   Receive imperfect context $\hat{x}_t$.
		\STATE   Select arm $a^{\dagger}_t$ as the UCB-optimal arm $A^{\dagger}_t\left(\hat{x}_t \right) $ defined in Definition \ref{def:ucboptarm}.
		\STATE  Observe the true context $x_t$ and the reward $y_t$
		\ENDFOR
	\end{algorithmic}
\end{algorithm}

\subsection{Proof of Lemma \ref{lma:KernelConcent}}\label{proof:concent}
\textbf{Lemma \ref{lma:KernelConcent}}\textit{
	Assume that the reward function $f(x,a)$ satisfies $\left| f(x,a)\right|\leq B$, the noise $n_t$ satisfies a sub-Gaussian distribution with parameter $b$,
	and  kernel ridge regression is used to estimate the reward function. With a probability of at least $1-\delta, \delta\in (0,1)$, for all $a\in \mathcal{A}$ and $t\in \mathbb{N}$, the estimation error satisfies
	$
	|\hat{f}_{t}(x,a)-f(x,a)|\leq h_ts_{t}(x,a),
	$
	where $h_t=\sqrt{\lambda}B+b\sqrt{\gamma_{t}-2\log\left( \delta\right) }$,   $\gamma_{t}=\log \det(\mathbf{I}+\mathbf{K}_{t}/\lambda)\leq \bar{d}\log (1+\frac{t}{\bar{d}\lambda} )$ and $\bar{d}$ is the rank of $\mathbf{K}_{t}$. Let $\mathbf{V}_t=\lambda \mathbf{I}+\sum_{s=1}^{t}\phi(x,a)\phi(x,a)^\top$, the squared confidence width is given by $ s^2_{t}(x,a)=\phi(x,a)^\top\mathbf{V}^{-1}_{t-1}\phi(x,a)=\frac{1}{\lambda}k([x,a],[x,a])-\frac{1}{\lambda}\mathbf{k}_{t}(x,a)^T(\mathbf{K}_{t}+\lambda \mathbf{I})^{-1}\mathbf{k}_{t}(x,a)$.
}
\begin{proof}
	Let $\phi: \left( \mathcal{X}\times \mathcal{A}\right)\rightarrow \mathcal{H}$ be the mapping function with respect to $k(\cdot,\cdot)$. Define $\Psi_{t}$ with the $s$th row as $\phi(x_{s},a_{s}), s=1,2,\cdots, t-1$ and thus $\mathbf{K}_{t}=\Psi_{t}\Psi_{t}^\top$.  Denote  $\mathbf{y}_t=[y_1,\cdots,y_{t-1}]^\top$ as the collected rewards and the noise vector $\mathbf{n}_t=[n_1,n_2\cdots,n_{t-1}]^\top$, so $\mathbf{y}_{t}=\Psi_{t}\theta+\mathbf{n}_t$. Then we have
	\begin{equation}\label{equ:estierror}
		\begin{split}
			|f(x,a)&-\hat{f}_{t}(x,a)| =|f(x,a)-\mathbf{k}_{t}^\top(x,a)(\mathbf{K}_{t}+\lambda \mathbf{I}_t)^{-1}\mathbf{y}_{t}|\\
			&=|f(x,a)-\mathbf{k}_{t}^\top(x,a)(\mathbf{K}_{t}+\lambda \mathbf{I}_t)^{-1}(\Psi_{t}\theta+\mathbf{n}_{t})|\\
			&\leq |f(x,a)-\mathbf{k}_{t}^\top(x,a)(\mathbf{K}_{t}+\lambda \mathbf{I}_t)^{-1}\Psi_{t}\theta|+\\
			& \qquad |\mathbf{k}_{t}^\top(x,a)(\mathbf{K}_{t}+\lambda \mathbf{I}_t)^{-1}\mathbf{n}_{t}|.
		\end{split}
	\end{equation}
	where $\mathbf{I}_t$ is an identity matrix with $(t-1)$ dimensions and the last inequality comes from triangle inequality.
	
	Let $\mathbf{V}_{t-1}=\Psi_{t}^\top\Psi_{t}+\lambda\mathbf{I}$ and define $s_{t}^2\left(x,a \right) =\phi(x,a)^\top\mathbf{V}_{t-1}^{-1}\phi(x,a)$. By Woodbury formula, we can write $s_{t}^2\left(x,a \right)$ by kernel functions:
	\begin{equation}
		\begin{split}
			s_{t}^2&\left(x,a \right)\! =\!\phi(x,a)^\top\!\!\left( \frac{1}{\lambda}\mathbf{I}\!-\!\frac{1}{\lambda}\Psi_t^\top\!\!\left( \lambda\mathbf{I}_t+\Psi_t\Psi_t^\top\right) ^{-1}\!\!\Psi_t\right)\!\! \phi(x,a)\\
			&=\frac{1}{\lambda}k([x,a],[x,a])-\frac{1}{\lambda}\mathbf{k}_{t}(x,a)^\top(\mathbf{K}_{t}+\lambda\mathbf{I}_t)^{-1}\mathbf{k}_{t}(x,a).
		\end{split}
	\end{equation}
	For the first term of Eqn.~\eqref{equ:estierror}, using Woodbury formula, we have
	\begin{equation}
		\begin{split}
			|f(x,a)&-\mathbf{k}_{t}^\top(x,a)(\mathbf{K}_{t}+\lambda \mathbf{I}_t)^{-1}\Psi_{t}\theta|\\
			&=|\phi(x,a)^\top\theta-\phi(x,a)^\top\Psi_{t}^\top(\Psi_{t}\Psi_{t}^\top+\lambda\mathbf{I}_t)^{-1}\Psi_{t}\theta|\\
			&=|\phi(x,a)^\top\theta-\phi(x,a)^\top(\Psi_{t}^\top\Psi_{t}+\lambda\mathbf{I})^{-1}\Psi_{t}^\top\Psi_{t}\theta|\\
			&=|\lambda\phi(x,a)^\top(\Psi_{t}^\top\Psi_{t}+\lambda\mathbf{I})^{-1}\theta|\leq \sqrt{\lambda}Bs_{t}(x,a),
		\end{split}
	\end{equation}
	where the inequality comes from Cauchy-Schwartz inequality and $\|\theta\|_{\mathbf{V}_{t-1}^{-1}}\leq \|\theta\|_2\mathrm{eig}_{\mathrm{max}}(\mathbf{V}_{t-1})\leq \frac{B}{\sqrt{\lambda}}$.
	
	For the second term, we have the following inequalities.
	\begin{equation}
		\begin{split}
			|\mathbf{k}_{t}^\top&(x,a)(\mathbf{K}_{t}+\lambda \mathbf{I}_t)^{-1}\mathbf{n}_{t}|=\\
			&\quad |\phi(x,a)^\top(\Psi_{t}^\top\Psi_{t}+\lambda\mathbf{I})^{-1}\Psi_{t}^\top\mathbf{n}_{t}|\\
			&\leq \|\phi(x,a)\|_{\mathbf{V}_{t-1}^{-1}}\sqrt{\mathbf{n}_{t}^\top\Psi_{t}(\Psi_{t}^\top\Psi_{t}+\lambda\mathbf{I})^{-1}\Psi_{t}^\top\mathbf{n}_{t}}\\
			& =s_{t}(x,a)\|\mathbf{n}_{t}\|_{\mathbf{K}_{t}\left(\mathbf{K}_{t}+\lambda\mathbf{I}_{t} \right)^{-1} },
		\end{split}
	\end{equation}
	where the first inequality comes from Cauchy-Schwartz inequality.
	
	Since $n_t$ satisfies $b$-sub-Gaussain distribution conditioned on $\mathcal{F}_{t-1}$, by Theorem 1 in \cite{yadkori2011}, with probability $1-\delta$, $\delta\in(0,1)$ for $t\in\mathbb{N}$ and $a\in \mathcal{A}$, we have
	\begin{equation}
		\begin{split}
			&\|\mathbf{n}_{t}\|^2_{\mathbf{K}_{t}\left(\mathbf{K}_{t}+\lambda\mathbf{I}_{t} \right)^{-1} }=	\|\Psi_t^T \mathbf{n}_t\|^2_{V_{t-1}^{-1}}\\
			&\leq 2b^2\log\left( \frac{\det(V_{t-1})^{1/2}}{\delta \det(\lambda\mathbf{I}_{d})^{1/2}}\right)=b^2\left( \gamma_{t}-2\log(\delta)\right) .
		\end{split}
	\end{equation}
	where $\gamma_{t}=\log\det\left( \mathbf{K}_{t}/\lambda+\mathbf{I}_t\right)  \leq \bar{d}\log (1+\frac{t}{\bar{d}\lambda} )$ where $\bar{d}$ is the dimension of $\mathbf{K}_t$ the inequality comes from Lemma 10 in \cite{yadkori2011}.
	
	Combining the bounds of the first and second term in Eqn.~\eqref{equ:estierror} and let $h_t=\sqrt{\lambda}B+b\sqrt{\gamma_t-2\log(\delta)}$, we obtain the concentration bound in Lemma 3.1:
	$|f(x,a)-\hat{f}_{t}(x,a)|\leq h_ts_t\left(x,a \right) $.
\end{proof}

\subsection{Proof of Lemma \ref{lma:sumvarbatch}}\label{proof:width}

\textbf{Lemma \ref{lma:sumvarbatch}}\textit{
	The confidence widths given 
	$x_t$ for $t\in\{1,\cdots, T\}$ satisfies $\sum_{t=1}^Ts^2_{t}(x_t,a_t)\!\leq 2\gamma_{T}$, where $\gamma_{T}=\log \det(\mathbf{I}+\mathbf{K}_{T}/\lambda)\leq \bar{d}\log (1+\frac{T}{\bar{d}\lambda} )$
	and $\bar{d}$ is the rank of $\mathbf{K}_{T}$, and so $\sum_{t=1}^Ts_{t}(x_t,a_t)\!\leq\sqrt{2T\gamma_{T}}$.
}
\begin{proof}
	By Lemma 11 in \cite{yadkori2011}, if $\lambda\geq 1$, we have
	\begin{equation}
		\begin{split}
			\sum_{t=1}^Ts^2_t(x_t,a_t)&=\sum_{t=1}^T\phi(x_t,a_t)^\top\mathbf{V}^{-1}_{t-1}\phi(x_t,a_t)\\
			&\leq 2\log\det\left( \frac{1}{\lambda}\mathbf{V}_T\right) =2\gamma_{T}
		\end{split}
	\end{equation}	
	where $\gamma_{T}=\log\det\left( \mathbf{K}_{T}/\lambda+\mathbf{I}_T\right)\leq \bar{d}\log (1+\frac{T}{\bar{d}\lambda} )$.
	By H\"older's inequality, we have
	\begin{equation}
		\sum_{t=1}^Ts_t(x_t,a_t)\leq\sqrt{T\sum_{t=1}^Ts^2_t(x_t,a_t)}=\sqrt{2T\gamma_T}.
	\end{equation}	
\end{proof}

\section{Proofs  Related to  \ouralgtwo}
\subsection{Proof of Theorem \ref{the:rewardmmucb}}
\textbf{Theorem \ref{the:rewardmmucb}}\textit{
	If \ouralgtwo is used to select arms with imperfect context, then for any true contexts $x_t\in \mathcal{B}_{\Delta}(\hat{x}_t)$ at round $t, t=1,\cdots, T$, with a probability of $1-\delta,\delta\in (0,1)$, we have the following lower bound on the worst-case cumulative reward
	\begin{equation}
		F_T\geq \sum_{t=1}^T MF_t-2h_T\sqrt{2T\bar{d}\log (1+\frac{T}{\bar{d}\lambda} )}
	\end{equation}
	where $MF_t$ is the optimal worst-case reward in Eqn.~\eqref{eqn:worstreward}, $\bar{d}$ is the rank of $\mathbf{K}_t$ and $h_T$ is given in Lemma \ref{lma:KernelConcent}.	 	}
\begin{proof}
	By Lemma \ref{lma:KernelConcent} and UCB defined in Definition \ref{def:defucb},
	with probability $\delta$, $\delta\in(0,1)$, the instantaneous regret can be bounded as
	\begin{equation}\label{equ:minucbbound}
		\begin{split}
			f(x_{t},a^{\textrm{I}}_t) \geq U_t\left(x_t,a^{\textrm{I}}_t \right)-2h_ts_t\left(x_t,a^{\textrm{I}}_t \right) \\
			\geq U_t\left(x^{\textrm{I}}_{t},a^{\textrm{I}}_t \right)-2h_ts_t\left(x_t,a^{\textrm{I}}_t \right)
		\end{split}
	\end{equation}
	where the last inequality comes from the arm selection policy of \ouralgtwo
	implying that  $x^{\textrm{I}}_{t}=\arg\min_{\mathbf{x} \in \mathcal{B}_{\Delta}(\hat{x}_t)}U_t\left(x,a^{\textrm{I}}_t \right)$.

	With the definition of the optimal worst-case reward $MF_t$ and exploiting the arm selection policy of \ouralgtwo, we can further bound the instantaneous reward as below.
	\begin{equation}\label{equ:rewardlowerbound}
		\begin{split}
			f(x_{t},&a^{\textrm{I}}_t)
			\geq \left[ U_t\left(x^{\textrm{I}}_{t},a^{\textrm{I}}_t \right)-2h_ts_t\left(x_t,a^{\textrm{I}}_t \right)\right] \\
			&\geq \left[ U_t\left(\bar{x}^{\textrm{I}}_{t},\bar{a}_t \right)-2h_ts_t\left(x_t,a^{\textrm{I}}_t \right)\right] \\
			&\geq \left[ U_t\left(\bar{x}^{\textrm{I}}_{t},\bar{a}_t \right)-2h_ts_t\left(x_t,a^{\textrm{I}}_t \right)\right]  -f\left(\bar{x}^{\textrm{I}}_{t},\bar{a}_t \right)+MF_t\\
			&\geq MF_t-2h_ts_t\left(x_t,a^{\mathrm{I}}_t  \right)
		\end{split}
	\end{equation}
	where $\bar{a}_t=\arg\max_{a\in\mathcal{A}}\min_{x\in\mathcal{B}_{\Delta}(\hat{x}_t)}f(x,a)$ is the optimal arm for maximizing the worst-case reward defined in Definition 1, $\bar{x}^{\textrm{I}}_{t}=\arg\min_{x \in \mathcal{B}_{\Delta}(\hat{x}_t)}U_t\left(x,\bar{a}_t\right)$,  the second inequality comes from the arm selection strategy
	of \ouralgtwo such that $\min_{x \in \mathcal{B}_{\Delta}(\hat{x}_t)}U_t\left(x,a^{\textrm{I}}_t\right)\geq \min_{x \in \mathcal{B}_{\Delta}(\hat{x}_t)}U_t\left(x,\bar{a}_t\right)$, the third inequality holds because the definition of $MF_t$ in Eqn. \eqref{eqn:worstreward}.
	which guarantees $MF_t=\min_{x\in\mathcal{B}_{\Delta}(\hat{x}_t)}f(x,\bar{a}_t)\leq f\left(\bar{x}^{\textrm{I}}_{t},\bar{a}_t \right)$, and the last inequality comes from Lemma \ref{lma:KernelConcent} 
	which guarantees $U_t\left(\bar{x}^{\textrm{I}}_{t},\bar{a}_t \right)\geq f\left(\bar{x}^{\textrm{I}}_{t},\bar{a}_t \right)$.
	Therefore, combined with Lemma~\ref{lma:sumvarbatch},
	we can get the bound the cumulative reward of \ouralgtwo.
\end{proof}
\subsection{Proof of Corollary \ref{the:regretmmucb}}
\textbf{Corollary \ref{the:regretmmucb}}\textit{
	If \ouralgtwo is used to select arms with imperfect context, then for any true contexts $x_t\in  \mathcal{B}_{\Delta}(\hat{x}_t)$ at round $t, t=1,\cdots, T$, with a probability of $1-\delta,\delta\in (0,1)$, we have the following bound on the cumulative true regret defined in Eqn.~\eqref{eqn:trueregret}:
	\begin{equation}
		\begin{split}
			R_T \!\leq\! \sum_{t=1}^T\overline{MR}_t +2h_T\sqrt{2T\bar{d}\log (1+\frac{T}{\bar{d}\lambda} )}
		\end{split}
	\end{equation}
	where $\overline{MR}_t=\max_{x \in \mathcal{B}_{\Delta}(\hat{x}_t)} f\left(x,A^*\left( x \right)  \right) -MF_t$, $MF_t$ is the optimal worst-case reward in Eqn.~\eqref{eqn:worstreward}, $\bar{d}$ is the rank of $\mathbf{K}_t$ and $h_T$ is given in Lemma \ref{lma:KernelConcent} .}
\begin{proof}
	For any $x_t\in \mathcal{B}_{\Delta}(\hat{x}_t)$, by Theorem \ref{the:rewardmmucb}, with a probability of $1-\delta,\delta\in (0,1)$, the instantaneous regret can be bounded as
	\begin{equation}\label{eqn:boundregretmmucb}
		\begin{split}
			r(x_t,& a^{\textrm{I}}_t ) =f\left(x_t, A^*(x_t) \right) -f\left(x_t, a^{\textrm{I}}_t \right) \\
			&\leq f\left(x_t, A^*(x_t) \right) -(MF_t-2h_ts_t\left(x_t,a^{\mathrm{I}}_t  \right))\\
			&\leq \max_{x\in\mathcal{B}_{\Delta}(\hat{x}_t)}\left\lbrace f\left(x, A^*(x) \right) \right\rbrace-MF_t+2h_ts_t\left(x_t,a^{\mathrm{I}}_t \right).
		\end{split}
	\end{equation}
	Corollary  \ref{the:regretmmucb} can be obtained by taking summation with $t$ and Lemma \ref{lma:sumvarbatch}.
\end{proof} 

\section{Proofs  Related to   \ouralgthree}

\subsection{Proof of Lemma \ref{lma:boudnsubopt}}
\textbf{Lemma \ref{lma:boudnsubopt}}
\textit{
	If \ouralgthree is used to select arms with imperfect context, then for each $t=1,2,\cdots, T$, with a probability at least $1-\delta,\delta\in (0,1)$, we have
	\begin{equation}
		\overline{D}_{a^{\mathrm{II}}_t,t}\leq MR_t+2h_ts_t\left(\dot{x}_t,  A^{\dagger}_t\left(\dot{x}_t \right) \right),
	\end{equation}
	where $MR_t$ is the optimal worst-case regret defined in Eqn.~\eqref{eqn:optworstreg}, $\dot{x}_t=\arg\max_{x\in\mathcal{B}_{\Delta}(\hat{x}_t)}D_{\tilde{a}_t}\left(x \right)$ is the context that maximizes the degradation given the arm $\tilde{a}_t$ defined for the optimal worst-case regret in Eqn.~\eqref{eqn:solutionmmreg}.
}
\begin{proof}
	Recall that in Eqn.~\eqref{eqn:minmaxregret}, the optimal arm for minimizing the worst-case regret is $\tilde{a}_t=\arg\min_{a\in\mathcal{A}}\max_{x\in\mathcal{B}_{\Delta}(\hat{x}_t)}r(x,a)$. By the arm selection policy of \ouralgthree, $a_t^{\textrm{II}}$ is th arm that minimizes $\overline{D}_{a,t}$, so the worst-case degradation can be bounded as follows.
	\begin{equation}
		\begin{split}
			\overline{D}_{a_t^{\textrm{II}},t}&\leq \overline{D}_{\tilde{a}_t,t}=\max_{x \in \mathcal{B}_{\Delta}(\hat{x}_t)}\left\lbrace U_t\left(x, A^{\dagger}_t\left( x\right)   \right) - U_t\left(x,\tilde{a}_t \right) \right\rbrace\\
			&=U_t\left(\dot{x}_{t},A^{\dagger}_t\left( \dot{x}_t\right)\right)- U_t\left(\dot{x}_{t} ,\tilde{a}_t  \right) \\
			&\leq U_t\left(\dot{x}_{t} ,A^{\dagger}_t\left( \dot{x}_t\right)\right)- U_t\left(\dot{x}_{t} ,\tilde{a}_t  \right)-\\
			&\qquad \left[ f\left(\dot{x}_{t} ,A^{\dagger}_t\left( \dot{x}_t\right)\right)-
			f\left(\dot{x}_{t} ,\tilde{a}_t\right)\right]+MR_t\\
			&\leq  2h_ts_t\left(\dot{x}_{t} ,A^{\dagger}_t\left( \dot{x}_t\right)\right) +MR_t,
		\end{split}
	\end{equation}
	where $\dot{x}_t=\arg\max_{x\in\mathcal{B}_{\Delta}(\hat{x}_t)}D_{\tilde{a}_t,t}(x)$,
	the second inequality holds because
	$\left[ f\left(\dot{x}_{t} ,A^{\dagger}_t\left( \dot{x}_t\right)\right)-
	f\left(\dot{x}_{t} ,\tilde{a}_t\right)\right]\leq\left[ f\left(\dot{x}_{t} ,A^*\left( \dot{x}_t\right)\right)-
	f\left(\dot{x}_{t},\tilde{a}_t\right)\right]=r(\dot{x}_t,\tilde{a}_t)\leq MR_t$
	where $MR_t=\max_{x\in\mathcal{B}_{\Delta}(\hat{x}_t)}r\left(x,\tilde{a}_t \right) $ is the optimal worst-case regret
	defined in Eqn.~\eqref{eqn:optworstreg}
	and the third inequality is from Lemma \ref{lma:KernelConcent}
	which guarantees that $U_t\left(\dot{x}_{t} ,A^{\dagger}_t\left( \dot{x}_t\right) \right)-f\left(\dot{x}_{t} ,A^{\dagger}_t\left( \dot{x}_t\right)\right)\leq 2h_ts_t\left(\dot{x}_{t} ,A^{\dagger}_t\left( \dot{x}_t\right)\right)$ and $f\left(\dot{x}_{t} ,\tilde{a}_t \right)-U_t\left(\dot{x}_{t} ,\tilde{a}_t \right)\leq 0$.
\end{proof}
\subsection{Proof of Proposition 5}
\label{sec:proofsumofconfidenceconsidered}
We first bound the confidence width sum of considered context-arm sequence in the Lemma \ref{lma:sumconfwdiscdiscrete} assuming a linear reward function, i.e. $\phi(x,a)=[x,a]$, and  context-arm space $\Phi=\mathcal{X}\times\mathcal{A}$ is finite with size $|\Phi|$. Then in Lemma \ref{lma:sumconfwdisccont}, we prove the bound for linear reward case and continuous context-arm space. Finally we generalize the bound to the kernel cases. For conciseness, denote $x_{a_t,t}=[x_t,a_t]\in \mathcal{R}^d$ as the true context and selected arm at round $t$ and $\dot{x}_{\dot{a}_t,t}=[\dot{x}_t,\dot{a}_t]\in \mathcal{R}^d$ as the considered context and arm at round $t$.

\begin{lemma}[Sum of Confidence Width with Finite Context-arm Space]\label{lma:sumconfwdiscdiscrete}
	Let $\mathcal{X}_T=\left\lbrace x_{a_1,1},\cdots,  x_{a_T,T}\right\rbrace$ be the sequence of true contexts and selected arms by bandit algorithms and $\dot{\mathcal{X}}_T=\left\lbrace  \dot{x}_{\dot{a}_1,1},\cdots,  \dot{x}_{\dot{a}_T,T} \right\rbrace $ be the considered sequence of contexts and arms. Suppose that both $x_{a_t,t}$  and $\dot{x}_{\dot{a}_t,t}$$\left( t=1,\cdots,T\right) $ belong to a finite set $\Phi$ with size $|\Phi|$. Besides, $\forall  \varphi\in \Phi$, $\exists \kappa\geq 0$,  two conditions are satisfied. First, $\exists t\leq \left \lceil  \kappa|\Phi|\right \rceil$ such that $x_{a_t,t}=\varphi$. Second, if at round $t$, $x_{a_t,t}=\varphi$, then $\exists t', t\leq t'\leq t+\left \lceil  \kappa|\Phi|\right \rceil$, such that $x_{a_{t'},t'}=\varphi$. The sum of confidence width is bounded as	
	\begin{equation}
		\sum_{t=1}^T s_{t}^2\left(\dot{x}_{\dot{a}_t,t} \right)\leq \left \lceil  \kappa|\Phi|\right \rceil\left( 2d\log\left(1+\frac{T}{d\lambda} \right)+\frac{1}{\lambda}	\right),
	\end{equation}
	where  $s_{t}^2\left(\dot{x}_{\dot{a}_t,t} \right)=\dot{x}^{\top}_{\dot{a}_t,t}\mathbf{V}^{-1}_{t-1}\dot{x}_{\dot{a}_t,t}$ and $\mathbf{V}_t=\lambda\mathbf{I}+\sum_{s=1}^tx_{a_s,s}x_{a_s,s}^{\top}$.
\end{lemma}

\begin{figure}[!t]	
	
	\centering
	\includegraphics[width={0.48\textwidth}]{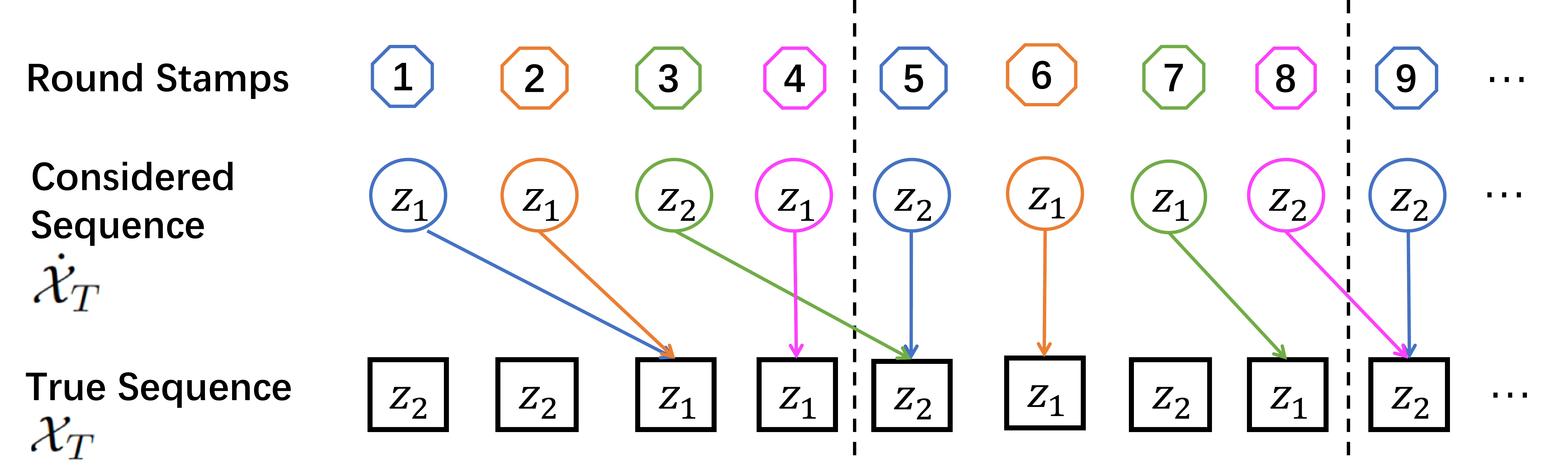}
	
	\caption{An example of sequence grouping in Lemma \ref{lma:sumconfwdiscdiscrete}. The feature set is $\Phi=\{z_1,z_2\}$. $\kappa=2$. $\left\lceil\kappa |\Phi|\right\rceil=4$ groups are constructed: $\mathcal{G}_1=\{1,5,9,\cdots\}, \mathcal{G}_2=\{2,6,\cdots\}, \mathcal{G}_3=\{3,7,\cdots\}, \mathcal{G}_4=\{4,8,\cdots\}$. The interval between the features in considered sequence and their counterparts in true sequence is less than $\left\lceil\kappa |\Phi|\right\rceil=4$.}
	\label{fig:sequence}
	\vspace{-0.4cm}	
\end{figure}

\begin{proof}
	
	Let $M= \left \lfloor  T/ \left \lceil  \kappa|\Phi|\right \rceil\right \rfloor$. Divide the round stamp sequence $\left\lbrace1, 2, \cdots, M\left \lceil  \kappa|\Phi|\right \rceil \right\rbrace $ uniformly into $\left \lceil  \kappa|\Phi|\right \rceil$ groups, each with $M$ elements. The $l$th group of round stamps, $\left( 1\leq l\leq \left \lceil  \kappa|\Phi|\right \rceil\right) $, is $\mathcal{G}_l=\left\lbrace l, l+\left \lceil  \kappa|\Phi|\right \rceil,\cdots,l+\left(M-1 \right)\left \lceil  \kappa|\Phi|\right \rceil \right\rbrace $, and $\mathcal{G}_{l,m}=\left\lbrace l, l+\left \lceil  \kappa|\Phi|\right \rceil,\cdots,l+(m-1)\left \lceil  \kappa|\Phi|\right \rceil\right\rbrace, (1\leq m\leq M)$  is a subset of  $\mathcal{G}_l$ with the first $m$ entries. A simple example of group construction is shown in Fig. \ref{fig:sequence}.
	
	Let $\mathbf{W}_{l,m}=\lambda\mathbf{I}_d+\sum_{s\in \mathcal{G}_{l,m}}\dot{x}_{\dot{a}_s,s}\dot{x}_{\dot{a}_s,s}^{\top}$, for $1\leq l\leq \left \lceil  \kappa|\Phi|\right \rceil$. For the sequence $\left\lbrace \dot{x}_{\dot{a}_t,t}  \mid t\in \mathcal{G}_l \right\rbrace$, by Lemma \ref{lma:sumvarbatch},
	we have $\sum_{t\in \mathcal{G}_l}\dot{x}^{\top}_{\dot{a}_t,t}\mathbf{W}^{-1}_{l,m_t-1}\dot{x}_{\dot{a}_t,t}\leq \gamma(\mathbf{W}_{l,M})$ where $m_t=\left \lfloor(t-1)/\left \lceil  \kappa|\Phi|\right \rceil\right \rfloor$ and $\gamma(\mathbf{W}_{l,M})=\log\frac{\det(\mathbf{W}_{l,M})}{\det(\lambda \mathbf{I})}$.
	
	Let's consider the sequence in $\mathcal{G}_l$. The two conditions in Lemma \ref{lma:sumconfwdisccont}, imply that for $t\in \mathcal{G}_{l}$ and its index $m_t=\left \lfloor(t-1)/\left \lceil  \kappa|\Phi|\right \rceil\right \rfloor$, $\forall s \in \mathcal{G}_{l,m_t-1}$, if $\dot{x}_{\dot{a}_s,s}=\varphi$, there exists $s'(s\leq s'< s+\left \lceil  \kappa|\Phi|\right \rceil)$ such that the true context-arm $x_{a_{s'},s'}=\varphi$ (For example, in Fig. \ref{fig:sequence}, $s=1$ is in $\mathcal{G}_{1}$, and there exists $s'=3\leq (4-1)$ such that $\dot{x}_{\dot{a}_s,s}=x_{a_s',s'}=z_1$). Therefore we have $s'\in \mathcal{T}_{t-1}=\left\lbrace 1,2,\cdots, t-1\right\rbrace $, and we conclude that $\left\lbrace \dot{x}_{\dot{a}_s,s}, s\in\mathcal{G}_{l,m_t-1} \right\rbrace \subseteq \left\lbrace x_{a_s,s}, s\in\mathcal{T}_{t-1} \right\rbrace$. Thus, considering $\mathbf{V}_{t-1}$ and $\mathbf{W}_{l,m_t-1}$ are both positive definite matrices, we have  $\mathbf{V}_{t-1}\succeq\mathbf{W}_{l,m_t-1}$  and $\mathbf{V}^{-1}_{m_t-1}\preceq \mathbf{W}^{-1}_{l,m_t-1}$.
	Therefore, for group $l$ ($1\leq l\leq \left \lceil  \kappa|\Phi|\right \rceil$), we have
	\begin{equation}
		\begin{split}
			\sum_{t\in \mathcal{G}_l}s^2_{t}\left(\dot{x}_{\dot{a}_t,t} \right) &=\sum_{t\in \mathcal{G}_l}\dot{x}^{\top}_{\dot{a}_t,t}\mathbf{V}^{-1}_{t-1}\dot{x}_{\dot{a}_t,t}\\
			&\leq \sum_{t\in \mathcal{G}_l}\dot{x}^{\top}_{\dot{a}_t,t}\mathbf{W}_{l,M}\dot{x}_{\dot{a}_t,t}\leq \gamma(\mathbf{W}_{l,M}).
		\end{split}
	\end{equation}
	
	From the above analysis, since $s^2_{t}\left(\dot{x}_{\dot{a}_t,t}\right) \leq \frac{1}{\lambda}$, we can get
	\begin{equation}\label{equ:sumvardiscrete}
		\begin{split}
			\sum_{t=1}^T s^2_{t}\left(\dot{x}_{\dot{a}_t,t} \right)& = \sum_{l=1}^{\left \lceil  \kappa|\Phi|\right \rceil} \sum_{t\in \mathcal{G}_l}s^2_{t}\left(\dot{x}_{\dot{a}_t,t} \right)+ \sum_{t=M\left \lceil  \kappa|\Phi|\right \rceil+1}^{T} s^2_{t}\left(\dot{x}_{\dot{a}_t,t} \right)\\
			&\leq \sum_{l=1}^{\left \lceil  \kappa|\Phi|\right \rceil} \gamma(\mathbf{W}_{l,M})+ \left \lceil  \kappa|\Phi|\right \rceil\frac{1}{\lambda}\\
			&\leq \left \lceil  \kappa|\Phi|\right \rceil\left( 2d\log\left(1+\frac{M}{d\lambda} \right)+\frac{1}{\lambda}	\right)\\
			&\leq \left \lceil  \kappa|\Phi|\right \rceil\left( 2d\log\left(1+\frac{T}{d\lambda} \right)+\frac{1}{\lambda}	\right),
		\end{split}
	\end{equation}
	where the second inequality comes from Lemma 11 in \cite{yadkori2011}.
	thus completing the proof.
\end{proof}

If the context-arm space is continuous, the finite set assumption of $\Phi$ in Lemma \ref{lma:sumconfwdiscdiscrete} is not satisfied anymore. To overcome this issue, we construct an $\epsilon$- covering $\Phi_{\epsilon}\in \Phi $ for the context-arm space $\Phi$ such that for each $\varphi\in\Phi$, there exists at least one $\bar{\varphi}\in\Phi_{\epsilon}$ satisfying $\|\varphi-\bar{\varphi} \|_2\leq \epsilon$. Since the dimension of $x_{a_t,t}$ is $d$, the size of the $\epsilon$- covering  is $|\Phi_{\epsilon}|\sim O\left( \frac{1}{\epsilon^d}\right) $. Now we can bound the sum of confidence width with linear reward function and continuous context-arm space in Lemma \ref{lma:sumconfwdisccont}.
\begin{lemma}[Sum of Confidence Width with Continuous Context-arm Space]\label{lma:sumconfwdisccont}
	
	Let $\mathcal{X}_T=\left\lbrace x_{a_1,1},\cdots,  x_{a_T,T}\right\rbrace$ be the sequence of true contexts and selected arms by bandit algorithms and $\dot{\mathcal{X}}_T=\left\lbrace  \dot{x}_{\dot{a}_1,1},\cdots,  \dot{x}_{\dot{a}_T,T} \right\rbrace $ be the considered sequence of contexts and actions. Suppose that both $x_{a_t,t}$  and $\dot{x}_{\dot{a}_t,t}$$\left( t=1,\cdots,T\right) $ belong to $\Phi$. Besides, with an $\epsilon-$ covering $\Phi_{\epsilon}\subseteq \Phi$, $\epsilon> 0$,  there exists $ \kappa \geq 0$ such that two conditions are satisfied for $\mathcal{X}_T$. First, $\forall \bar{\varphi}\in \Phi_{\epsilon}$, $\exists t\leq \left \lceil  \kappa/\epsilon^d\right \rceil$ such that $x_{a_t,t}\in \mathcal{C}_{\epsilon}\left(\bar{\varphi} \right) $. Second, if at round $t$, $x_{a_t,t}\in \mathcal{C}_{\epsilon}\left(\bar{\varphi} \right) $ for some $\bar{\varphi}\in \Phi_{\epsilon}$, then $\exists  t\leq t'< t+\left \lceil  \kappa/\epsilon^d\right \rceil$ such that $x_{a_t',t'}\in \mathcal{C}_{\epsilon}\left(\bar{\varphi} \right) $. The sum of squared confidence width is bounded as
	\begin{equation}
		\begin{split}
			\sum_{t=1}^T s^2_{t}\left(\dot{x}_{\dot{a}_t,t}\right)
			\!\!\leq\! \!\sqrt{T}\left(\!\! 4d\log\!\left(\!1+\frac{T}{d\lambda} \!\right)\!+\!\frac{1}{\lambda}\!\!	\right) \!+\!\frac{8\kappa^{2/d}}{\lambda}T^{1-1/d},
		\end{split}
	\end{equation}
	where $s_{t}^2\left(\dot{x}_{\dot{a}_t,t} \right)=\dot{x}^{\top}_{\dot{a}_t,t}\mathbf{V}^{-1}_{t-1}\dot{x}_{\dot{a}_t,t}$ and $\mathbf{V}_t=\lambda\mathbf{I}_d+\sum_{s=1}^tx_{a_s,s}x_{a_s,s}^{\top}$.
\end{lemma}
\begin{proof}

	Based on the $\epsilon$- covering $\Phi_{\epsilon}\in \Phi $, we divide the round stamp sequence $\left\lbrace1, 2, \cdots, M\left \lceil  \kappa/\epsilon^d\right \rceil \right\rbrace $ uniformly into $\left \lceil  \kappa/\epsilon^d\right \rceil$ groups, each with $M= \left \lfloor  T/ \left \lceil  \kappa/\epsilon^d\right \rceil\right \rfloor$ elements. The $l$th group $\left( 1\leq l\leq \left \lceil  \kappa/\epsilon^d\right \rceil\right) $ is $\mathcal{G}_l=\left\lbrace l, l+\left \lceil  \kappa/\epsilon^d\right \rceil,\cdots,l+\left(M-1 \right)\left \lceil  \kappa/\epsilon^d\right \rceil \right\rbrace $, and $\mathcal{G}_{l,m}=\left\lbrace l, l+\left \lceil  \kappa/\epsilon^d\right \rceil,\cdots,l+\left(m-1 \right)\left \lceil  \kappa/\epsilon^d \right \rceil\right\rbrace$ is a subset of  $\mathcal{G}_l$ with the first $m$ entries. 
	
	Now we consider the $l$th group. The two conditions in this lemma imply that
	for a certain $t\in \mathcal{G}_{l}$ and its index $m_t=\left \lfloor (t-1)/ \left \lceil  \kappa/\epsilon^d \right \rceil\right \rfloor$, $\forall s \in \mathcal{G}_{l,m_t-1}$, if $\dot{x}_{\dot{a}_s,s}\in \mathcal{C}_\epsilon\left( \bar{\varphi}\right) $ for some $\bar{\varphi}\in\Phi_{\epsilon} $, there exists $s'(s\leq s'< s+\left \lceil  \kappa/\epsilon^d\right \rceil)$ such that the true context-arm $x_{a_{s'},s'}\in \mathcal{C}_\epsilon\left( \bar{\varphi}\right) $. Denote $s'=\zeta(s)$ mapping $s$ to its corresponding $s'$, and we have $\zeta(s)\leq t-1$.
	Also, denote $\mathcal{G}'_{l,m_t-1}=\left\lbrace \zeta(s)\mid s\in \mathcal{G}_{l,m_t-1}\right\rbrace $, and we have $\mathcal{G}'_{l,m_t-1}\subseteq \mathcal{T}_{t-1}=\left\lbrace 1,2,\cdots, t-1\right\rbrace $.
	Let $\mathbf{W}_{l,m}=\lambda\mathbf{I}_d+\sum_{s\in \mathcal{G}_{l,m}}x_{a_{\zeta(s)},\zeta(s)}x_{a_{\zeta(s)},\zeta(s)}^{\top}=\lambda\mathbf{I}_d+\sum_{s'\in \mathcal{G}'_{l,m}}x_{a_{s'},s'}x_{a_{s'},s'}^{\top}$, for $1\leq l\leq \left \lceil  \kappa/\epsilon^d\right \rceil$. By the above analysis, we have  $\mathbf{V}_{t-1}\succeq\mathbf{W}_{l,m_t-1}$, and so $\mathbf{V}^{-1}_{t-1}\preceq \mathbf{W}^{-1}_{l,m_t-1}$ considering $\mathbf{V}_{t-1}$ and $\mathbf{W}_{l,m_t-1}$ are both positive definite matrices.
	
	By Lemma 2, for the sequence $\{x^{\top}_{a_{\zeta(t)},\zeta(t)}, t\in\mathcal{G}_l\}$,
	we have $\sum_{t\in \mathcal{G}_l}x^{\top}_{a_{\zeta(t)},\zeta(t)}\mathbf{W}^{-1}_{l,m_t-1}x_{a_{\zeta(t)},\zeta(t)}\leq \gamma(\mathbf{W}_{l,M})$.
	Denote $e_{t}=\dot{x}_{\dot{a}_t,t}-x_{a_{\zeta(t)},\zeta(t)}$.  Since $x_{a_{\zeta(t)},\zeta(t)}$ and $\dot{x}_{\dot{a}_{t},t}$ belong to the same cell in $\epsilon-$ covering $\Phi_{\epsilon}$, we have $\|e_{t}\|\leq 2\epsilon$ by triangle inequality.
	Therefore for group $\mathcal{G}_{l}$, we have
	\begin{equation}
		\begin{split}
			\sum_{t\in \mathcal{G}_{l}}&s^2_{t}\left(\dot{x}_{\dot{a}_t,t}\right) =\sum_{t\in \mathcal{G}_{l}}\left( \dot{x}_{\dot{a}_t,t}\right)^\top \mathbf{V}^{-1}_{t-1}\left( \dot{x}_{\dot{a}_t,t}\right)\\
			&=\sum_{t\in \mathcal{G}_{l}}\left( x_{a_{\zeta(t)},\zeta(t)}+e_{t}\right)^\top \mathbf{V}^{-1}_{t-1}\left( x_{a_{\zeta(t)},\zeta(t)}+e_{t}\right)\\
			&\leq 2\sum_{t\in \mathcal{G}_{l}}x^{\top}_{a_{\zeta(t)},\zeta(t)}\mathbf{V}^{-1}_{t-1}x_{a_{\zeta(t)},\zeta(t)}+2\sum_{t\in \mathcal{G}_{l}}e^{\top}_{t}\mathbf{V}^{-1}_{t-1}e_{t}\\
			&\leq 2\sum_{t\in \mathcal{G}_{l}}x^{\top}_{a_{\zeta(t)},\zeta(t)}\mathbf{W}^{-1}_{l,m_t-1}x_{a_{\zeta(t)},\zeta(t)}+2\frac{M}{\lambda}\left( 2\epsilon\right) ^2\\
			&\leq 2\gamma(\mathbf{W}_{l,M}) +8\frac{M}{\lambda}\epsilon^2,	
		\end{split}
	\end{equation}
	where the first inequality comes from Cauchy-Schwartz inequality and the second inequality holds because  $\mathbf{V}^{-1}_{t-1}\preceq \mathbf{W}^{-1}_{l,m_t-1}$ and $\mathrm{eig}(\mathbf{V}^{-1}_{t-1})\leq \frac{1}{\lambda}$, and the last inequality holds by Lemma 2.
	
	Now we can bound the sum of the confidence widths for the whole sequence as
	\begin{equation}
		\begin{split}
			\sum_{t=1}^T &s^2_{t}\left(\dot{x}_{\dot{a}_t,t}\right) = \sum_{l=1}^{\left \lceil  \kappa/\epsilon^d\right \rceil} \sum_{t\in \mathcal{G}_l}s^2_{t}\left(\dot{x}_{\dot{a}_t,t} \right)+ \!\!\!\sum_{t=M\left \lceil  \kappa/\epsilon^d\right \rceil+1}^{T}\!\!\!\! s^2_{t}\left(\dot{x}_{\dot{a}_t,t} \right)\\
			&\leq \left \lceil  \kappa/\epsilon^d\right \rceil\left(2\gamma(\mathbf{W}_{l,M}) +\frac{8M}{\lambda}\epsilon^2+\frac{1}{\lambda}	\right)\\
			&\leq \left \lceil  \kappa/\epsilon^d\right \rceil\left( 4d\log\left(1+\frac{T}{d\lambda} \right)+\frac{1}{\lambda}	\right)+\frac{8T}{\lambda}\epsilon^2,
		\end{split}
	\end{equation}
	where the last inequality holds because $\gamma(\mathbf{W}_{l,M})\leq 2d\log\left(1+\frac{M}{d\lambda} \right) $ by Lemma 11 in \cite{yadkori2011} and $M= \left \lfloor  T/ \left \lceil  \kappa/\epsilon^d\right \rceil\right \rfloor$.
	Let $\left \lceil  \kappa/\epsilon^d\right \rceil=T^{1/2} $, i.e. $\epsilon=\left( \kappa^2/T\right)^{1/2d} $, then we have
	\begin{equation}
		\begin{split}
			\sum_{t=1}^T s^2_{t}\left(\dot{x}_{\dot{a}_t,t}\right)
			\!\!\leq\!\! \sqrt{T}\left(\! 4d\log\!\left(1+\frac{T}{d\lambda} \right)\!+\!\frac{1}{\lambda}	\right) \!\!+\!\!\frac{8\kappa^{2/d}}{\lambda}T^{1-1/d},
		\end{split}
	\end{equation}
	thus completing the proof.
\end{proof}

Now we can prove proposition \ref{prop:sumconfwd} by generalizing Lemma \ref{lma:sumconfwdisccont} into kernel case assuming the mapping function $\phi:\mathcal{X}\times\mathcal{A}\rightarrow \mathcal{H}$ is Lipschitz continuous with constant $L_{\phi}$, i.e. $\forall x, y\in \mathcal{X}\times\mathcal{A}$, $\|\phi(x)-\phi(y)\|\leq L_{\phi}\|x-y\|$.

\textbf{Proof of Proposition \ref{prop:sumconfwd}}

\textbf{Proposition \ref{prop:sumconfwd}}\textit{
	Let $\mathcal{X}_T=\left\lbrace x_{a_1,1},\cdots,  x_{a_T,T}\right\rbrace$ be the sequence of true contexts and selected arms by bandit algorithms and $\dot{\mathcal{X}}_T=\left\lbrace  \dot{x}_{\dot{a}_1,1},\cdots,  \dot{x}_{\dot{a}_T,T} \right\rbrace $ be the considered sequence of contexts and actions. Suppose that both $x_{a_t,t}$  and $\dot{x}_{\dot{a}_t,t}$belong to $\Phi$. Besides, with an $\epsilon-$ covering $\Phi_{\epsilon}\subseteq \Phi$, $\epsilon> 0$,  there exists $ \kappa \geq 0$ such that two conditions are satisfied. First, $\forall \bar{\varphi}\in \Phi_{\epsilon}$, $\exists t\leq \left \lceil  \kappa/\epsilon^d\right \rceil$ such that $x_{a_t,t}\in \mathcal{C}_{\epsilon}\left(\bar{\varphi} \right) $. Second, if at round $t$, $x_{a_t,t}\in \mathcal{C}_{\epsilon}\left(\bar{\varphi} \right) $ for some $\bar{\varphi}\in \Phi_{\epsilon}$, then $\exists  t\leq t'< t+\left \lceil  \kappa/\epsilon^d\right \rceil$ such that $x_{a_t',t'}\in \mathcal{C}_{\epsilon}\left(\bar{\varphi} \right) $. If the mapping function $\phi$ is Lipschitz continuous with constant $L_{\phi}$, the sum of squared confidence widths is bounded as
	\begin{equation}
		\begin{split}
			\sum_{t=1}^T \!s^2_{t}\left(\dot{x}_{\dot{a}_t,t}\right)
			\!\!\leq \!\!\sqrt{T}\!\left(\! 4\tilde{d}\log\left(\!1\!+\!\frac{T}{\tilde{d}\lambda} \!\right)\!+\!\frac{1}{\lambda}\!	\right) \!\!+\!\!\frac{\!8L^2_{\phi}\kappa^{2/d}\!}{\lambda}T^{1-1/d},
		\end{split}
	\end{equation}
	where $d$ is the dimension of $x_{a_t,t}$, $\tilde{d}$ is the effective dimension defined in the proof, $s^2_{t}\left(\dot{x}_{\dot{a}_t,t}\right)\!=\!\phi(\dot{x}_{\dot{a}_t,t})^{\top}\mathbf{V}^{-1}_{t-1}\phi(\dot{x}_{\dot{a}_t,t})$ and $\mathbf{V}_t\!=\!\lambda\mathbf{I}+\sum_{s=1}^t\phi(x_{a_s,s})\phi(x_{a_s,s})^{\top}$.}
\begin{proof}
	With the same method in Lemma \ref{lma:sumconfwdisccont},  we construct $\left \lceil  \kappa/\epsilon^d\right \rceil$ groups, each with $M= \left \lfloor  T/ \left \lceil  \kappa/\epsilon^d\right \rceil\right \rfloor$ elements.  Define $\mathcal{G}_l, \mathcal{G}_{l,m}, \mathcal{G}'_{l,m}$,$\zeta(t), m_t, e_t$ same as Lemma \ref{lma:sumconfwdisccont} and let $\mathbf{W}_{l,m}=\lambda\mathbf{I}+\sum_{s\in \mathcal{G}_{l,m}}\phi(x_{a_{\zeta(s)},\zeta(s)})\phi(x_{a_{\zeta(s)},\zeta(s)})^{\top}=\lambda\mathbf{I}+\sum_{s'\in \mathcal{G}'_{l,m}}\phi(x_{a_{s'},s'})\phi(x_{a_{s'},s'})^{\top}$, for $1\leq l\leq \left \lceil  \kappa/\epsilon^d\right \rceil$. By Lemma \ref{lma:sumconfwdisccont}, we have  $\mathbf{V}_{t-1}\succeq\mathbf{W}_{l,m_t-1}$, and so $\mathbf{V}^{-1}_{t-1}\preceq \mathbf{W}^{-1}_{l,m_t-1}$. 
	Therefore for group $\mathcal{G}_{l}$, we have
	\begin{equation}
		\begin{split}
			&\sum_{t\in \mathcal{G}_{l}}s^2_{t}\left(\dot{x}_{\dot{a}_t,t}\right)=\sum_{t\in \mathcal{G}_{l}}\phi\left( \dot{x}_{\dot{a}_t,t}\right)^\top \mathbf{V}^{-1}_{t-1}\phi\left( \dot{x}_{\dot{a}_t,t}\right)\\
			&=\sum_{t\in \mathcal{G}_{l}}\|\phi\left( x_{a_{\zeta(t)},\zeta{t}}\right)+\phi\left( \dot{x}_{\dot{a}_t,t}\right) -\phi\left( x_{a_{\zeta(t)},\zeta{t}}\right)\|^2_{\mathbf{V}^{-1}_{t-1}}\\
			&\leq \!2\!\sum_{t\in \mathcal{G}_{l}}\!\!\|\phi\!\left(\! x_{a_{\zeta(t)},\zeta{t}}\right)\!\!\|^2_{\mathbf{V}^{-1}_{t-1}}\!\!\!\!+\!\!2\sum_{t\in \mathcal{G}_{l}}\!\!\|\phi\left(\! \dot{x}_{\dot{a}_t,t}\!\right)\! \!-\!\phi\left( x_{a_{\zeta(t)},\zeta{t}}\right)\!\!\|^2_{\mathbf{V}^{-1}_{t-1}}\\
			&\leq 2\sum_{t\in \mathcal{G}_{l}}\|\phi(x_{a_{\zeta(t)},\zeta(t)})\|^2_{\mathbf{W}^{-1}_{l,m_t-1}}+2\frac{M}{\lambda}\left( L_{\phi}\|e_t\|_2\right) ^2\\
			&\leq 2\gamma(\mathbf{W}_{l,M}) +8\frac{ML^2_{\phi}}{\lambda}\epsilon^2,	
		\end{split}
	\end{equation}
	where the first inequality comes from Cauchy-Schwartz inequality and the second inequality holds because  $\mathbf{V}^{-1}_{t-1}\preceq \mathbf{W}^{-1}_{l,m_t-1}$, $\mathrm{eig}(\mathbf{V}^{-1}_{t-1})\leq \frac{1}{\lambda}$ and Lipschitz continuity of $\phi$ such that $\|\phi\left( \dot{x}_{\dot{a}_t,t}\right) -\phi\left( x_{a_{\zeta(t)},\zeta{t}}\right)\|_2\leq L_{\phi}\|e_t\|_2$, and the last inequality holds by Lemma 2.
	
	For kernel case, with $\mathbf{K}_{l,M}\in\mathcal{R}^{M\times M}$ containing $k(x_{a_{\zeta(t)},\zeta(t)},x_{a_{\zeta(t)},\zeta(t)})$ for $t\in \mathcal{G}_l$, we have $\gamma(\mathbf{W}_{l,M})=\log\frac{\det(\mathbf{W}_{l,M})}{\det(\lambda \mathbf{I}_{\mathcal{H}})}=\log\frac{\det(\mathbf{I}_{M}+\mathbf{K}_{l,M})}{\det(\lambda \mathbf{I}_M)}\leq 2\tilde{d}_l\log(1+\frac{M}{\tilde{d}_l})$, where $\tilde{d}_l$ is the rank of $\mathbf{K}_{l,M}$. Define the effective dimension as $\tilde{d}=\arg\max_{\tilde{d}_l=\tilde{d}_1, \cdots,\tilde{d}_{\left \lceil  \kappa/\epsilon^d\right \rceil}}2\tilde{d}_l\log(1+\frac{M}{\tilde{d}_l})$. Then we can bound the sum of the confidence widths for the whole sequence as
	\begin{equation}
		\begin{split}
			\sum_{t=1}^T &s^2_{t}\left(\dot{x}_{\dot{a}_t,t}\right)= \sum_{l=1}^{\left \lceil  \kappa/\epsilon^d\right \rceil} \sum_{t\in \mathcal{G}_l}s^2_{t}\left(\dot{x}_{\dot{a}_t,t} \right)+ \!\!\!\!\sum_{t=M\left \lceil  \kappa/\epsilon^d\right \rceil+1}^{T}\!\!\!\! s^2_{t}\left(\dot{x}_{\dot{a}_t,t} \right)\\
			&\leq \left \lceil  \kappa/\epsilon^d\right \rceil\left(2\gamma(\mathbf{W}_{l,M}) +\frac{8ML^2_{\phi}}{\lambda}\epsilon^2+\frac{1}{\lambda}	\right)\\
			&\leq \left \lceil  \kappa/\epsilon^d\right \rceil\left( 4\tilde{d}\log(1+\frac{M}{\tilde{d}})+\frac{1}{\lambda}	\right)+\frac{8TL^2_{\phi}}{\lambda}\epsilon^2,
		\end{split}
	\end{equation}
	Let $\left \lceil  \kappa/\epsilon^d\right \rceil=T^{1/2} $, i.e. $\epsilon=\left( \kappa^2/T\right)^{1/2d} $, then we have
	\begin{equation}
		\begin{split}
			\sum_{t=1}^T s^2_{t}\left(\dot{x}_{\dot{a}_t,t}\right)
			\!\!\leq\!\! \sqrt{T}\!\!\left( \!\!4\tilde{d}\log\!\left(\!\!1\!+\!\frac{T}{\tilde{d}\lambda} \right)\!+\!\frac{1}{\lambda}	\right) \!\!+\!\!\frac{8L^2_{\phi}\kappa^{2/d}}{\lambda}T^{1-1/d},
		\end{split}
	\end{equation}
	thus completing the proof.
\end{proof}

\subsection{Proof of Theorem \ref{the:regretminwd}}
\textbf{Theorem \ref{the:regretminwd}.}\textit{
	If \ouralgthree is used to select arms with imperfect context and as time goes on, and the conditions in Proposition \ref{prop:sumconfwd} are satisfied, then for any true context $x_t\in \mathcal{B}_{\Delta}(\hat{x}_t)$ at round $t, t=1,\cdots, T$,  with a probability of $1-\delta,\delta\in (0,1)$, we have the following bound on the cumulative true regret:
	\begin{equation*}
		\begin{split}
			R_T \leq &\sum_{t}^{T}MR_t+2h_TT^{\frac{3}{4}}\sqrt{\left( 4\tilde{d}\log\left(1+\frac{T}{\tilde{d}\lambda} \right)	+\frac{1}{\lambda} 	\right)} +\\
			&4\sqrt{\frac{2}{\lambda}}L_{\phi}\kappa^{\frac{1}{d}}h_TT^{1-\frac{1}{2d}}+2h_T\!\!\sqrt{2T\bar{d}\log (1\!+\!\frac{T}{\bar{d}\lambda} )},
		\end{split}
	\end{equation*}
	where $MR_t$ is the optimal worst-case regret for round $t$ in Eqn.~\eqref{eqn:optworstreg}, $d$ is the dimension of $x_{a_t,t}$, $\tilde{d}$ is the effective dimension defined in the proof of Proposition \ref{prop:sumconfwd}, $\bar{d}$ is the rank of $\mathbf{K}_t$ and $h_T$ is given in Lemma \ref{lma:KernelConcent}.}

\begin{proof}
	Since with a probability $1-\delta$, $\delta\in(0,1)$, $r\left(x_t, a^{\textrm{II}}_t \right)\leq \overline{D}_{a_t^{\textrm{II}},t}+2h_ts_t(x_t,a_t^{\textrm{II}}) $, we can bound the cumulative regret of \ouralgthree as
	\begin{equation}
		\begin{split}
			R_T&= \sum_{t=1}^{T}r\left(x_t,a^{\textrm{II}}_t \right)\leq \sum_{t=1}^T\left[ \overline{D}_{a_t^{\textrm{II}},t}+2h_ts_t(x_t,a_t^{\textrm{II}})\right] \\
			&\leq \sum_{t=1}^T\left[ MR_t+2h_ts_t\left(\dot{x}_t,  A^{\dagger}_t\left(\dot{x}_t \right) \right)+2h_ts_t(x_t,a_t^{\textrm{II}})\right],
		\end{split}
	\end{equation}
	where the second inequality comes from Lemma \ref{lma:boudnsubopt}.
	
	Since the sequence $\{x_t,a^{\mathrm{II}}_t \}$ meets the conditions in Proposition 5, by replacing $\dot{x}_{\dot{a}_t,t}$ and $x_{a_{t},t}$ in Proposition \ref{prop:sumconfwd} with $[\dot{x}_t,A^{\dagger}_t\left(\dot{x}_t \right) ]$ and $[x_t,a^{\mathrm{II}}_t]$, we have
	\begin{equation}
		\begin{split}
			\sum_{t=1}^T&s_t\left(\dot{x}_t,  A^{\dagger}_t\left(\dot{x}_t \right) \right)\leq \sqrt{T\sum_{t=1}^Ts^2_t\left(\dot{x}_t,  A^{\dagger}_t\left(\dot{x}_t \right) \right)}\\
			&\leq \sqrt{T^{3/2}\left( \!\!4\tilde{d}\log\!\left(\!\!1\!+\!\frac{T}{\tilde{d}\lambda} \right)\!+\!\frac{1}{\lambda}	\right) +\frac{8L^2_{\phi}\kappa^{2/d}}{\lambda}T^{2-1/d}}\\
			&\leq T^{\frac{3}{4}}\sqrt{\!\!4\tilde{d}\log\!\left(\!\!1\!+\!\frac{T}{\tilde{d}\lambda} \right)\!+\!\frac{1}{\lambda}} +2\sqrt{\frac{2}{\lambda}}L_{\phi}\kappa^{\frac{1}{d}}T^{1-\frac{1}{2d}}.
		\end{split}
	\end{equation}
	Combining with Lemma \ref{lma:sumvarbatch}, we can get the bound in Theorem \ref{the:regretminwd}.
\end{proof}

\subsection{Proof of Corollary \ref{cor:rewardminwd}}
\textbf{Corollary \ref{cor:rewardminwd}}\textit{
	If \ouralgthree is used to select arms with imperfect context and as time goes on, and the true sequence of context and arm obeys the conditions in Proposition \ref{prop:sumconfwd}, then for any true contexts $x_t\in \mathcal{B}_{\Delta}(\hat{x}_t)$ at round $t, t=1,\cdots, T$, with a probability of $1-\delta,\delta\in (0,1)$, 	we have the following lower bound of the cumulative reward
	\begin{equation*}
		\begin{split}
			F_T &\geq \!\sum_{t=1}^{T}\left[ MF_t\!-\!MR_t\right]\!-\!2h_TT^{\frac{3}{4}}\!\sqrt{\!\left( 4\tilde{d}\log\left(1+\frac{T}{\tilde{d}\lambda} \right)\!\!	+\!\!\frac{1}{\lambda} 	\right)}\\
			&-4\sqrt{\frac{2}{\lambda}}L_{\phi}\kappa^{\frac{1}{d}}h_TT^{1-\frac{1}{2d}}-2h_T\!\!\sqrt{2T\bar{d}\log (1\!+\!\frac{T}{\bar{d}\lambda} )},
		\end{split}
	\end{equation*}
	where $MR_t$ is the optimal worst-case regret for round $t$ in Eqn.~\eqref{eqn:optworstreg}, $d$ is the dimension of $x_{a_t,t}$, $\tilde{d}$ is the effective dimension defined in the proof of Proposition \ref{prop:sumconfwd}, $\bar{d}$ is the rank of $\mathbf{K}_t$, and $h_T$ is given in Lemma \ref{lma:KernelConcent}.}
\begin{proof}
	By Lemma \ref{lma:KernelConcent},
	with a probability $1-\delta$, $\delta\in[0,1]$, we can bound the reward as below
	\begin{equation}
		\begin{split}
			&f(x_{t},a^{\textrm{II}}_t)
			\geq U_t\left(x_{t},a^{\textrm{II}}_t \right)-2h_ts_t\left(x_t,a^{\textrm{II}}_t \right) \\
			=&U_t\left(x_{t},a^{\textrm{II}}_t \right)\!-\!U_t\left(x_{t},A^{\dagger}_t(x_t)  \right)\!+\\
			&\qquad U_t\left(x_{t},A^{\dagger}_t(x_t)  \right)\!\!-\!\!2h_ts_t\left(x_t,a^{\textrm{II}}_t \right)\\
			=&-D_{a^{\textrm{II}}_t,t}\left(x_t \right) +U_t\left(x_{t},A^{\dagger}_t(x_t)  \right)-2h_ts_t\left(x_t,a^{\textrm{II}}_t \right),
		\end{split}
	\end{equation}
	Recall that the upper bound of the worst-case degradation is defined as $\overline{D}_{a,t}\!\!=\!\!\max_{x \in \mathcal{B}_{\Delta}(\hat{x}_t)}\left\lbrace  U_t\left(x, A^{\dagger}_t(x)  \right) - U_t\left(x,a\right) \right\rbrace$, so the reward can be further bounded as
	\begin{equation}\label{eqn:boundrewardalgorithm3}
		\begin{split}
			&f(x_{t},a^{\textrm{II}}_t)\\
			\geq&
			\left( -\overline{D}_{a^{\textrm{II}}_t,t}\right)+U_t\left(x_{t},A^{\dagger}_t(x_t)  \right)-2h_ts_t\left(x_t,a^{\textrm{II}}_t \right)\\
			\geq&
			\left( -\overline{D}_{a^{\textrm{II}}_t,t}\right)+\min_{x\in \mathcal{B}_{\Delta}(\hat{x}_t)}U_t\left(x,A^{\dagger}_t(x)  \right)-2h_ts_t\left(x_t,a^{\textrm{II}}_t \right)\\
			\geq&
			\left( -\overline{D}_{a^{\textrm{II}}_t,t}\right)+\max_{a\in \mathcal{A}}\min_{x\in \mathcal{B}_{\Delta}(\hat{x}_t)}U_t\left(x,a \right)-2h_ts_t\left(x_t,a^{\textrm{II}}_t \right),
		\end{split}
	\end{equation}
	where the last inequality is because $\min_{x\in \mathcal{B}_{\Delta}(\hat{x}_t)}U_t\left(x,A^{\dagger}_t(x)  \right)\!\!=\!\!\min_{\mathbf{x}\in \mathcal{B}_{\Delta}(\hat{x}_t)} \max_{a\in \mathcal{A}}\!\!U_t\left(x,a  \right)\geq \max_{a\in \mathcal{A}}\min_{x\in \mathcal{B}_{\Delta}(\hat{x}_t)}U_t\left(x,a \right)$ by max-min inequality.
	
	Note that $\max_{a\in \mathcal{A}}\min_{x\in \mathcal{B}_{\Delta}(\hat{x}_t)}U_t\left(x,a \right)$ is the arm selection policy of \ouralgtwo whose solutions are $a^{\textrm{I}}_t$ and $x^{\textrm{I}}_t$. This observation is important since it bridges  \ouralgthree and \ouralgtwo. Also, recall that in Eqn.~\eqref{eqn:maxminreward}, $ \bar{a}_t=\arg\max_{a \in \mathcal{A}}\min_{x \in \mathcal{B}_{\Delta}(\hat{x}_t)}f\left(x,a \right)$ is the optimal arm for worst-case reward and $MF_t=\min_{x \in \mathcal{B}_{\Delta}(\hat{x}_t)}f\left(x,\bar{a}_t \right)$ is the optimal worst-case reward.
	Then following Eqn.~\eqref{eqn:boundrewardalgorithm3}, we have
	\begin{equation}
		\begin{split}
			&f(x_{t},a^{\textrm{II}}_t)\\
			\geq &\left( -\overline{D}_{a^{\textrm{II}}_t,t}\right)+\max_{a\in \mathcal{A}}\min_{x\in \mathcal{B}_{\Delta}(\hat{x}_t)}U_t\left(x,a \right)-2h_ts_t\left(x_t,a^{\textrm{II}}_t \right)\\
			=& \left( -\overline{D}_{a^{\textrm{II}}_t,t}\right)+ U_t\left(x^{\textrm{I}}_t,a^{\textrm{I}}_t \right)-2h_ts_t\left(x_t,a^{\textrm{II}}_t \right) \\
			\geq &\left( -\overline{D}_{a^{\textrm{II}}_t,t}\right)+ U_t\left(\bar{x}^{\textrm{I}}_t,\bar{a}_t \right)-2h_ts_t\left(x_t,a^{\textrm{II}}_t \right) \\
			\geq &\left(\! -\overline{D}_{a^{\textrm{II}}_t,t}\!\right)\!+\! U_t\left(\bar{x}^{\textrm{I}}_t,\bar{a}_t \right)\!-\!2h_ts_t\left(x_t,a^{\textrm{II}}_t \right)\!-\!f\left(\bar{x}^{\textrm{I}}_t,\bar{a}_t \right)\!+\!MF_t\\
			\geq &  \left( -\overline{D}_{a^{\textrm{II}}_t,t}\right)+MF_t-2h_ts_t\left(x_t,a^{\textrm{II}}_t \right),
		\end{split}
	\end{equation}
	where $\bar{x}^{\textrm{I}}_{t}=\arg\min_{x \in \mathcal{B}_{\Delta}(\hat{x}_t)}U_t\left(x,\bar{a}_t\right)$, the second inequality holds by the arm selection strategy
	of \ouralgtwo such that $U_t\left(x^{\textrm{I}}_t,a^{\textrm{I}}_t \right)\geq U_t\left(\bar{x}^{\textrm{I}}_t,\bar{a}_t \right)$, the third inequality comes from the definition of $MF_t$ in Eqn.~\eqref{eqn:worstreward} which guarantees  $MF_t=\min_{x\in\mathcal{B}_{\Delta}(\hat{x}_t)}f(x,\bar{a}_t)\leq f\left(\bar{x}^{\textrm{I}}_t,\bar{a}_t \right)$ and the forth inequality comes from Lemma \ref{lma:KernelConcent} 
	such that $U_t\left(\bar{x}^{\textrm{I}}_t,\bar{a}_t \right)\geq f\left(\bar{x}^{\textrm{I}}_t,\bar{a}_t \right)$.
	
	Finally, since $\overline{D}_{a^{\textrm{II}}_t,t}+2h_ts_t\left(x_t,a^{\textrm{II}}_t \right)$ is the upper bound of instantaneous regret of \ouralgthree, by directly using Theorem \ref{the:regretminwd}
	we can prove the lower bound reward of \ouralgthree.	
\end{proof}
\end{document}